\documentclass{article}
 
\usepackage{arxiv}
 
\usepackage[utf8]{inputenc}
\usepackage[T1]{fontenc}
\usepackage{hyperref}
\usepackage{url}
\usepackage{booktabs}
\usepackage{amsfonts}
\usepackage{nicefrac}
\usepackage{microtype}
\usepackage{xcolor}
\usepackage{graphicx}
\usepackage{subcaption}
\usepackage{enumitem}
\usepackage{multirow}
\usepackage{algorithm}
\usepackage{algorithmic}
\usepackage{tcolorbox}
\usepackage{amsmath}
\usepackage{fontawesome5}
\usepackage[table]{xcolor} 
\usepackage{amssymb} 
 
\usepackage{amsthm}
\theoremstyle{plain}
\newtheorem{theorem}{Theorem}[section]
\newtheorem{proposition}[theorem]{Proposition}
\newtheorem{lemma}[theorem]{Lemma}
\newtheorem{corollary}[theorem]{Corollary}
\theoremstyle{definition}
\newtheorem{definition}[theorem]{Definition}
\newtheorem{assumption}[theorem]{Assumption}
\theoremstyle{remark}

\usepackage{tcolorbox}
\tcbuselibrary{skins}

\definecolor{pipogray}{RGB}{248, 248, 248}
 
\tcolorboxenvironment{theorem}{
    enhanced,
    colback=blue!3!white,      
    colframe=blue!60!black,    
    arc=4pt,
    boxrule=0.8pt,
    left=6pt, right=6pt,
    top=6pt, bottom=6pt,
    boxsep=0pt,
    drop fuzzy shadow=gray!20
}
 
\tcolorboxenvironment{proposition}{
    enhanced,
    colback=teal!3!white,      
    colframe=teal!60!black,    
    arc=4pt,
    boxrule=0.8pt,
    left=6pt, right=6pt,
    top=6pt, bottom=6pt,
    boxsep=0pt,
    drop fuzzy shadow=gray!20
}
 
\tcolorboxenvironment{corollary}{
    enhanced,
    colback=violet!3!white,    
    colframe=violet!60!black,  
    arc=4pt,
    boxrule=0.8pt,
    left=6pt, right=6pt,
    top=6pt, bottom=6pt,
    boxsep=0pt,
    drop fuzzy shadow=gray!20
}
 
\renewcommand{\undertitle}{} 
\renewcommand{\headeright}{}

\date{}
 
\makeatletter
\renewcommand{\@maketitle}{%
  \vbox{%
    \hsize\textwidth
    \linewidth\hsize
    \vskip 0.1in
    \@toptitlebar
    \centering
    {\LARGE \@title\par}
    \@bottomtitlebar
    \vskip -0.1in
    \def\And{%
      \end{tabular}\hfil\linebreak[0]\hfil%
      \begin{tabular}[t]{c}\bf\rule{\z@}{24\p@}\ignorespaces%
    }
    \def\AND{%
      \end{tabular}\hfil\linebreak[4]\hfil%
      \begin{tabular}[t]{c}\bf\rule{\z@}{24\p@}\ignorespaces%
    }
    \begin{tabular}[t]{c}\bf\rule{\z@}{24\p@}\@author\end{tabular}%
  \vskip 0.4in \@minus 0.1in \center{\@date}   \vskip 0.2in
  }
}
\makeatother

\title{Policy Improvement Reinforcement Learning}

\author{%
  Huaiyang Wang\textsuperscript{1,3,\thanks{\raggedright Equal Contribution. \textsuperscript{\dag}Corresponding Author. For any questions, please feel free to contact: \texttt{huaiyangwang@buaa.edu.cn}, \texttt{yikunb@buaa.edu.cn}.}},~
  Xiaojie Li\textsuperscript{1,*},~
  Xiaohan Wang\textsuperscript{3},~
  Zhixia Zhang\textsuperscript{1},~
  Xiaodong Lu\textsuperscript{1,3},~
  Zixuan Huang\textsuperscript{1},~\\
  \textbf{Jiajun Chai}\textsuperscript{3},~
  \textbf{Guojun Yin}\textsuperscript{3},~
  \textbf{Deqing Wang}\textsuperscript{1},~
  \textbf{Haoyi Zhou}\textsuperscript{1},~
  \textbf{Yaodong Yang}\textsuperscript{2},~
  \textbf{Jianxin Li}\textsuperscript{1},~
  \textbf{Yikun Ban}\textsuperscript{1,\dag}\\[6pt]
  \textsuperscript{1}\textbf{\textit{Beihang University}}\qquad
  \textsuperscript{2}\textbf{\textit{Peking University}}\qquad
  \textsuperscript{3}\textbf{\textit{Meituan}}\\[4pt]
  \href{https://jacckma.github.io/pirl/}{\faGithub\ \texttt{Github Page: jacckma.github.io/pirl/}}
}

\begin{document}

\maketitle

\vspace{-1cm}

\begin{abstract}

Reinforcement learning has become a central post-training paradigm for improving LLM and agent capabilities. Yet existing RL post-training methods share a common blind spot: they update policies from rewards, advantages, or distillation targets computed on the current batch, without observing the effect of the update itself. This open-loop optimization can make local signals an unreliable proxy for actual policy improvement, especially under finite sampling, generation stochasticity, and feedback noise.
We argue that the missing ingredient is policy-improvement feedback: the ability to measure progress across policy iterations. We introduce Policy Improvement Reinforcement Learning (PIRL), which formulates expected performance gain between successive policies as an explicit objective aligned with final task performance.
Building on PIRL, we propose Policy Improvement Policy Optimization (PIPO), a plug-in closed-loop framework that verifies the previous update against a sliding-window historical performance anchor. PIPO uses this feedback to modulate the local learning signal of the base policy optimization algorithm, reinforcing updates associated with measured progress and suppressing those associated with performance drops.
We provide theoretical evidence that PIPO locally aligns policy updates with the PIRL improvement objective. Experiments on mathematical reasoning, code, tool-use, and self-distillation settings show that PIPO yields consistent gains across PPO, group-relative, and self-distillation policy optimization families.
\end{abstract}

\vspace{-0.5cm}

\begin{figure}[H]
  \centering
  \captionsetup{font=small,skip=4pt}
  \vspace{0.5cm}

  \begin{minipage}{\textwidth}
  
    \begin{subfigure}[b]{0.68\textwidth}
      \centering
      \includegraphics[width=\textwidth]{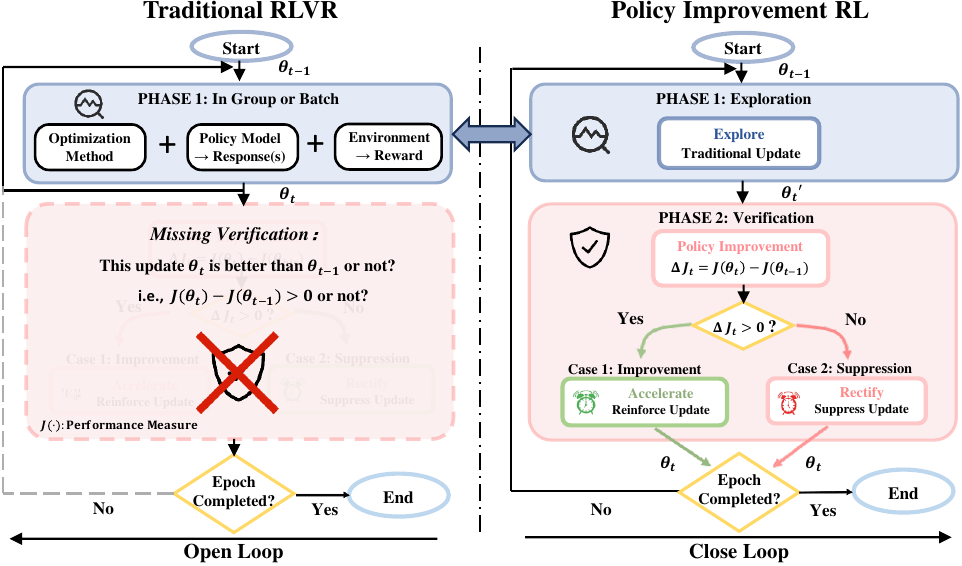}
    \end{subfigure}
    \hfill
    \begin{subfigure}[b]{0.3\textwidth}
      \centering
      \includegraphics[width=\textwidth]{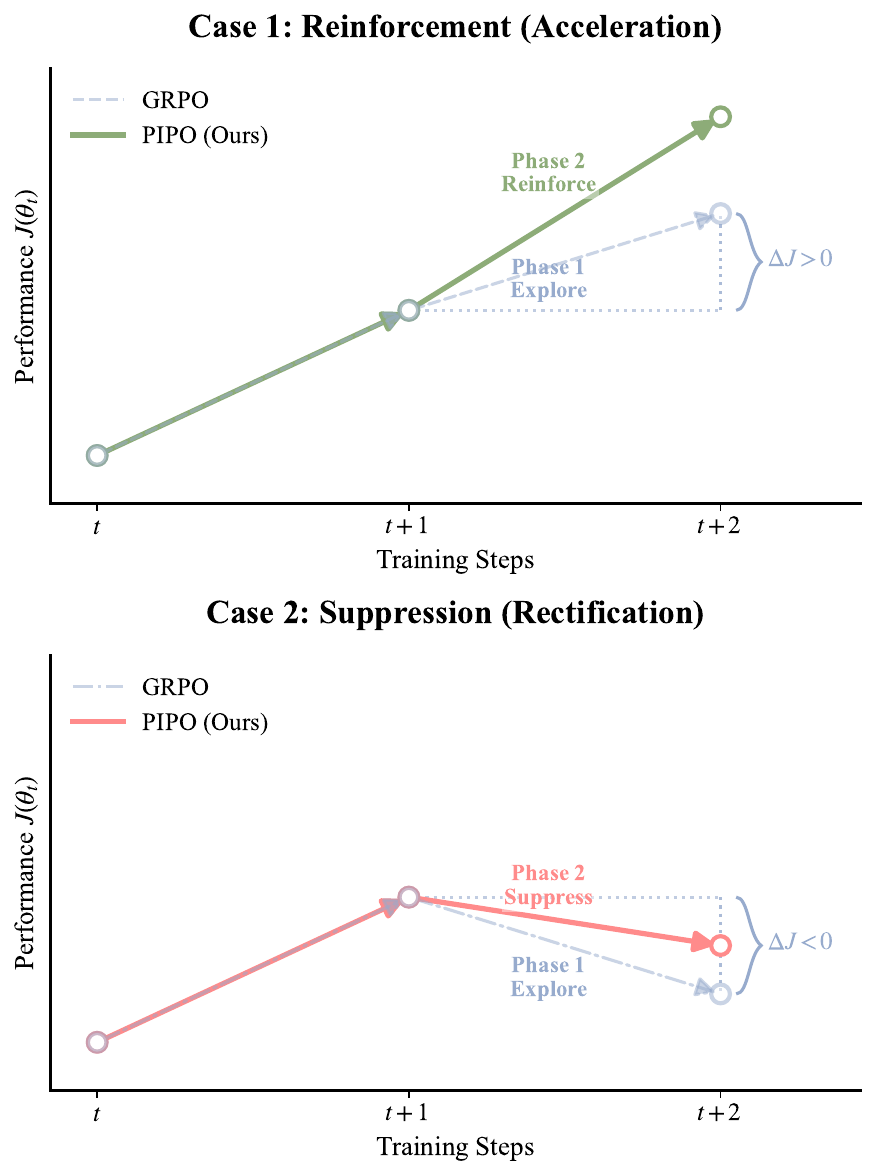}
    \end{subfigure}
    
  \end{minipage}

  \caption{\textbf{Overview of Policy Improvement Reinforcement Learning (PIRL).}
  \textbf{Left}: Existing RL post-training methods often follow an open-loop paradigm, updating policies from instantaneous rewards without verifying actual improvement.
  \textbf{Middle}: PIRL introduces a verification stage, forming a closed-loop optimization driven by policy improvement signals.
  \textbf{Right}: During verification, updates are adaptively regulated: positive signals ($\Delta J > 0$) are amplified, while negative signals ($\Delta J < 0$) trigger rectification to suppress harmful updates and stabilize training.}
  \label{fig:pipo}
  \vspace{-0.25cm}
  
\end{figure}

\section{Introduction}

Reinforcement learning has become a central post-training paradigm for improving the reasoning, coding, and tool-use capabilities of large language models and agents \cite{ban2026agent,yang2025qwen3,Guo_2025,zou2025transformer,chen2026scoring,xie2026uniarm}. A prominent line of work is reinforcement learning with verifiable rewards (RLVR), where objective rewards from answer verifiers, program execution, or tool feedback provide scalable supervision for domains such as mathematical reasoning \cite{chen2026weakdriven,li2026counterfactual} and program synthesis \cite{wu2025ucoder,zuo2026how}. Recent progress has produced several families of policy optimization methods. PPO remains a classic proximal policy optimization algorithm \cite{schulman2017proximal}. Group-relative methods such as GRPO \cite{shao2024deepseekmath} and its extensions remove the critic by normalizing rewards within groups of sampled responses.

Beyond scalar rewards, reinforcement learning with rich feedback (RLRF) uses richer feedback such as execution errors, judge comments, or tool-use traces to construct denser learning signals. A major representative direction in RLRF is on-policy distillation (OPD), where a student learns from teacher supervision on its own sampled trajectories \cite{song2026survey,li2026rethinking}. Self-distillation methods further specialize this idea by instantiating the teacher and student from the same model under different contexts \cite{hubotter2026reinforcement,zhao2026self}.

Despite their different local learning signals, these methods share a common blind spot: they update the policy from the rewards, advantages, or distillation targets computed on the current batch, without ever observing the effect of the update itself---a paradigm we refer to as \emph{open-loop} RL optimization. Improving such a local objective on sampled trajectories does not guarantee that the resulting policy is better than its predecessor~\cite{yue2025impact,wu2026distribution}, and with no mechanism to detect or correct a bad step, optimization can silently drift or collapse. Finite sampling, generation stochasticity, and feedback noise widen this gap further, making the local signal an unreliable proxy for the true effect of the updat~\cite{yang2026your,chen2026exploration,cheng2025entropy}e.

We argue that the missing ingredient is \emph{policy-improvement feedback}: a signal, revealed only by subsequent batches, that measures inter-iteration progress directly. RL post-training should not treat each update as an isolated local optimization problem, but should also ask whether a policy transition produces measurable improvement. We formalize this principle as \textbf{Policy Improvement Reinforcement Learning (PIRL)}, which directly optimizes the expected performance gain between successive policies. Because cumulative inter-iteration improvement remains aligned with final policy performance, PIRL offers a unified view that shifts RL post-training from local-signal optimization toward sustained policy improvement.

Building on PIRL, we propose \textbf{Policy Improvement Policy Optimization (PIPO)}, a plug-in \emph{closed-loop} policy optimization framework based on inter-batch policy-improvement verification. At each iteration, PIPO compares current empirical performance against a sliding-window historical anchor to evaluate whether the previous update improved the policy, and uses this feedback to modulate the local learning signal of the base policy optimization algorithm: improvements are reinforced, while performance drops are suppressed. In this way, PPO, group-relative optimization, and self-distillation optimization methods can be transformed from open-loop updates driven solely by local signals into closed-loop optimization with policy-improvement feedback.

Our contributions can be summarized as follows: (1) \textbf{New Paradigm:} We identify the absence of policy improvement feedback as a common limitation across RL post-training methods, and introduce \emph{Policy Improvement Reinforcement Learning}, which treats policy improvement as an explicit objective aligned with final performance. (2) \textbf{General Algorithm:} We propose PIPO, a unified closed-loop framework that can be plugged into PPO, group-relative methods such as GRPO, GSPO, and DAPO, and self-distillation methods such as SDPO. (3) \textbf{Theoretical and Empirical Evidence:} We analyze how policy-improvement feedback aligns with local policy updates, study group-relative optimization as an illustrative sensitivity case, and show empirically that PIPO yields consistent gains across mathematical reasoning, code, tool-use, and self-distillation settings.

\section{Preliminaries}
\label{sec:preliminaries}

We first define a standard RL post-training objective. For a query $q \sim \mathcal{D}$, a policy $\pi_\theta(\cdot \mid q)$ generates a response $y$, and the task provides a reward $R(q,y)$. The policy maximizes the expected reward:
\begin{equation}
\label{eq:rl_objective}
J_{\mathrm{RL}}(\theta)
= \mathbb{E}_{q \sim \mathcal{D}} \left[ \mathbb{E}_{y \sim \pi_\theta(\cdot \mid q)} R(q,y) \right].
\end{equation}
For a fixed query $q$, let
\begin{equation*}
p(q;\theta) = \mathbb{E}_{y \sim \pi_\theta(\cdot \mid q)}[R(q,y)],
\end{equation*}
so the ideal RL gradient is
\begin{equation*}
g_{\mathrm{ideal}} = \nabla_\theta J_{\mathrm{RL}}(\theta) = \mathbb{E}_{q \sim \mathcal{D}} \left[ \nabla_\theta p(q;\theta) \right].
\end{equation*}

We next introduce three representative policy optimization families considered in this work:

\paragraph{Proximal Policy Optimization.}
PPO uses a clipped trust-region surrogate. For a sampled response $y$, a common shaped reward is
\begin{equation*}
\widetilde{R}(q,y)
=
R(q,y)
-
\beta \log \frac{\pi_{\theta}(y \mid q)}{\pi_{\mathrm{ref}}(y \mid q)},
\end{equation*}
where $\beta$ controls reference-policy regularization. PPO estimates token-level advantages with generalized advantage estimation (GAE):
\begin{equation}
\label{eq:ppo_advantage_std}
A_{\tau}^{\mathrm{PPO}}
=
\sum_{\ell \ge 0}(\gamma\lambda_{\mathrm{GAE}})^\ell
\delta_{\tau+\ell},
\quad
\delta_{\tau}=r_{\tau}+\gamma V_\psi(s_{\tau+1})-V_\psi(s_{\tau}).
\end{equation}
where $s_\tau=(q,y_{<\tau})$ denotes the decoding state and $r_\tau$ denotes the shaped token-level reward. Given this advantage, PPO optimizes
\begin{equation*}
\mathcal{J}_{\mathrm{PPO}}(\theta)
=
\mathbb{E}_{q \sim \mathcal{D},\, y \sim \pi_{\theta_{\mathrm{old}}}(\cdot \mid q)}
\left[
\frac{1}{|y|}
\sum_{\tau=1}^{|y|}
\min\left(
\rho_\tau(\theta) A_\tau^{\mathrm{PPO}},
\mathrm{clip}(\rho_\tau(\theta), 1-\epsilon, 1+\epsilon) A_\tau^{\mathrm{PPO}}
\right)
\right],
\end{equation*}
where $\rho_\tau(\theta)=\frac{\pi_\theta(y_\tau \mid q,y_{<\tau})}{\pi_{\theta_{\mathrm{old}}}(y_\tau \mid q,y_{<\tau})}$ is the importance-sampling ratio.

\paragraph{Group-Relative Policy Optimization.}
For group-relative optimization, we use GRPO as the representative instantiation. For each query $q$, GRPO samples $G$ responses $y_i \sim \pi_{\theta_{\mathrm{old}}}(\cdot \mid q)$ with rewards $R_i := R(q,y_i)$, and constructs a critic-free advantage by normalizing rewards within the group:
\begin{equation}
\label{eq:grpo_advantage_std}
A_i^{\mathrm{GRPO}}=\frac{R_i-\mu_q}{\sigma_q}, \quad
\mu_q=\frac{1}{G}\sum_{i=1}^G R_i,\quad
\sigma_q=\sqrt{\frac{1}{G}\sum_{i=1}^G(R_i-\mu_q)^2}.
\end{equation}
The GRPO objective is
\begin{equation*}
\mathcal{J}_{\text{GRPO}}(\theta)
=
\mathbb{E}_{q \sim \mathcal{D},\, \{y_i\}_{i=1}^{G} \sim \pi_{\theta_{\mathrm{old}}}(\cdot \mid q)^G}
\bigg[
\frac{1}{G} \sum_{i=1}^G
\frac{1}{|y_i|}\sum_{\tau=1}^{|y_i|}
\min \Big(
\rho_{i,\tau}(\theta) A_i^{\mathrm{GRPO}},
\text{clip}(\rho_{i,\tau}(\theta), 1-\epsilon, 1+\epsilon) A_i^{\mathrm{GRPO}}
\Big)
\bigg],
\end{equation*}
where $\rho_{i,\tau}(\theta)=\frac{\pi_\theta(y_{i,\tau} \mid q,y_{i,<\tau})}{\pi_{\theta_{\mathrm{old}}}(y_{i,\tau} \mid q,y_{i,<\tau})}$ is the importance-sampling ratio. The response-level advantage $A_i^{\mathrm{GRPO}}$ is broadcast to all tokens in response $y_i$.

\paragraph{Self-Distillation Policy Optimization.}
We use SDPO~\cite{hubotter2026reinforcement} as our self-distillation baseline. Given a query $q$, the current policy first acts as a student and generates a response $y \sim \pi_\theta(\cdot \mid q)$. After receiving privileged information $f$, such as execution errors, a successful rollout, or a ground-truth solution, the same policy is re-evaluated as a feedback-conditioned self-teacher. Let $q_\theta(\cdot \mid q,f,y_{<\tau})$ denote this self-teacher distribution. SDPO distills the self-teacher into the student by minimizing a token-level reverse KL objective:
\begin{equation*}
\mathcal{L}_{\mathrm{SDPO}}(\theta)
=
\mathbb{E}_{q,y,f}
\left[
\sum_{\tau=1}^{|y|}
D_{\mathrm{KL}}
\left(
\pi_{\theta}(\cdot \mid q,y_{<\tau})
\;\|\;
\mathrm{stopgrad}\!\left(q_{\theta}(\cdot \mid q,f,y_{<\tau})\right)
\right)
\right],
\end{equation*}
which provides dense token-level attribution through the self-teacher/student log-probability difference:
\begin{equation}
\label{eq:sdpo_attribution_std}
A_{\tau}^{\mathrm{SDPO}}(y_\tau)
=
\log q_{\theta}(y_\tau \mid q,f,y_{<\tau})
-
\log \pi_{\theta}(y_\tau \mid q,y_{<\tau}).
\end{equation}

Taken together, PPO, GRPO, and SDPO differ in how they construct local attribution: GAE advantages, group-relative advantages, and self-teacher log-ratio advantages, respectively. However, after applying the update, these algorithms do not explicitly check whether optimizing the sampled trajectories and scores has produced a better policy. This shared open-loop structure motivates PIRL, which treats policy improvement as an explicit feedback signal for controlling policy learning.

\section{Policy Improvement Reinforcement Learning}
\label{sec:pirl}

Existing RL post-training algorithms differ in how they assign local credit, but most of them optimize each update in isolation. To address this limitation, we introduce \emph{policy improvement} as a feedback mechanism, shifting the optimization paradigm from open-loop updates to \emph{closed-loop} inter-iteration optimization. Specifically, we propose Policy Improvement Reinforcement Learning (PIRL), a learning framework that directly optimizes performance gain between successive policies. Unlike conventional surrogate objectives derived from instantaneous reward, PIRL leverages inter-iteration progress as its primary learning signal. We begin by formalizing the measurement of this improvement.

\begin{definition}[Policy Improvement]
\label{def:policy_improvement}
Let $J_{\mathrm{RL}}(\theta)$ denote the expected performance defined in Eq.~\eqref{eq:rl_objective}. For a given policy update from $\theta_{t}$ to $\theta_{t+1}$, we define the \emph{policy improvement} at iteration $t$ as the step-wise performance gain:
\begin{equation}
\label{eq:policy_improvement}
\Delta J_t := J_{\mathrm{RL}}(\theta_{t+1}) - J_{\mathrm{RL}}(\theta_{t}).
\end{equation}
\end{definition}

Rather than optimizing an isolated surrogate objective at each iteration, PIRL seeks a trajectory of policies that achieves consistent improvement over time.

\begin{definition}[PIRL Objective]
\label{def:pirl_objective}
Consider an optimization process over $T$ iterations. The PIRL objective seeks a sequence of policies $\{\theta_t\}_{t=1}^T$ that maximizes the cumulative expected policy improvement:
\begin{equation}
\label{eq:pirl_objective}
\max_{\{\theta_t\}_{t=1}^T} \sum_{t=1}^T \mathbb{E}\!\left[ \Delta J_t \right].
\end{equation}
\end{definition}

We now demonstrate that this objective is structurally consistent with the global RL objective.

\begin{theorem}
[Objective Alignment]
\label{prop:pirl_alignment}
For a fixed initialization $\theta_0$, maximizing the PIRL objective exactly maximizes the final policy performance: 
\[
\arg\max_{\{\theta_t\}_{t=1}^T} \sum_{t=1}^T \mathbb{E}\!\left[ \Delta J_t \right] \;=\; \arg\max_{\theta_T} J_{\mathrm{RL}}(\theta_T).
\]
\end{theorem}

As demonstrated by Theorem \ref{prop:pirl_alignment}, explicitly maximizing $\Delta J_t$ gives a temporal view of RL that remains aligned with final task performance. PIRL therefore provides the principle behind closed-loop optimization: the policy-learning procedure should not only ask which samples receive positive credit under a local surrogate, but also whether the resulting policy transition produces measurable progress. 

\section{Policy Improvement Policy Optimization}
\label{sec:pipo_algorithm}

Building upon the theoretical foundation of PIRL, we present Policy Improvement Policy Optimization (PIPO), a practical algorithm that closes the loop around policy updates. PIPO keeps the local credit assignment of a base RL algorithm, but adds a global verification signal that measures whether the previous policy transition produced a better policy in empirical evaluation. This feedback does not replace local credit assignment; it verifies the effect of the previous update. This yields a plug-in mechanism for PPO, group-relative policy optimization, and self-distillation policy optimization.

\subsection{Policy Improvement Reward}
\label{subsec:improvement_reward}

We now describe how PIPO turns retrospective policy-improvement feedback into a training signal for the previous policy transition. The key idea is to evaluate a policy update using trajectories sampled from the updated policy itself: if the new policy produces higher verifier rewards after the transition, the preceding update is treated as useful evidence for policy improvement. At iteration $t$, PIPO estimates the instantaneous policy performance $J_{\mathrm{RL}}(\theta_t)$ using the mean verifier reward $\mu_t$ over a freshly sampled batch $\mathcal{B}_t$:

\begin{equation}
\label{eq:mu_t}
\mu_t
:=
\frac{1}{|\mathcal{B}_t|}
\sum_{(q,y)\in \mathcal{B}_t} R(q,y),
\qquad
y \sim \pi_{\theta_{t}}(\cdot \mid q).
\end{equation}

Here, $\mu_t$ serves as empirical evidence for the quality of the policy update from the previous iteration. To obtain a more robust reference for recent policy performance, PIPO compares $\mu_t$ with a sliding-window historical anchor. We compute the historical mean $\mu_{\mathrm{his}}$ and standard deviation $\sigma_{\mathrm{his}}$ over the past $K$ iterations as:

\begin{equation}
\label{eq:global_stats}
\mu_{\mathrm{his}} = \frac{1}{K} \sum_{k=1}^K \mu_{t-k}, \qquad
\sigma_{\mathrm{his}} = \sqrt{\frac{1}{K-1} \sum_{k=1}^K (\mu_{t-k} - \mu_{\mathrm{his}})^2}.
\end{equation}

The difference $\mu_t - \mu_{\mathrm{his}}$ provides a practical estimate of whether the latest policy transition improves over recent empirical performance. We normalize this empirical improvement by the historical standard deviation to obtain the \emph{standardized improvement feedback}:
\begin{equation}
\label{eq:xi_def}
\xi_t := \frac{\mu_t - \mu_{\mathrm{his}}}{\sigma_{\mathrm{his}}}.
\end{equation}
Specifically, $\xi_t > 0$ provides evidence that the preceding policy update produced a genuine performance gain and should be reinforced. Conversely, $\xi_t \le 0$ indicates a possible performance decline and triggers the closed-loop mechanism to suppress or rectify the update. Since this scalar feedback does not specify sample- or token-level credit, PIPO uses the base algorithm's attribution scores as the local basis for modulation.

\begin{definition}[Policy Improvement Reward]
\label{def:pi_reward}
For each historical learning unit $u_{t-1,i}$ generated during the previous iteration $t-1$ (e.g., a sampled response for PPO or group-relative methods, or a token/step for self-distillation), let $a_{t-1,i}$ denote its local attribution score. The \emph{Policy Improvement (PI) Reward} evaluated at iteration $t$ is constructed as:
\begin{equation}
\label{eq:meta_reward_def}
\hat{r}^{\mathrm{PI}}_{t,i}
:=
\underbrace{a_{t-1,i}}_{\text{Local Attribution}}
\;\cdot\;
\underbrace{\phi_{\lambda}(\xi_t)}_{\text{Improvement Feedback}},
\end{equation}

where the choice of $a_{t-1,i}$ depends on the base RL algorithm. We map $\xi_t$ to a modulation coefficient using the following rectification function $\phi_{\lambda}: \mathbb{R} \rightarrow \mathbb{R}$:

\begin{equation*}
\phi_{\lambda}(x) = 
\begin{cases} 
x & \text{if } x \ge 0, \\
\lambda \cdot x & \text{if } x < 0.
\end{cases}
\end{equation*}

This design treats positive and negative improvement feedback asymmetrically: positive feedback keeps its original magnitude, while negative feedback is scaled by $\lambda$ to control the strength of correction and preserve part of the exploratory information from the previous update. We empirically validate this design choice in Section~\ref{sec:ablation}.
\end{definition}

This construction separates two roles: $a_{t-1,i}$ preserves the base algorithm's local credit assignment, while $\phi_{\lambda}(\xi_t)$ decides whether the previous update should be reinforced or suppressed after observing its empirical effect.

For PPO and SDPO, PIPO directly uses the original attribution scores: $a_{t-1,\tau}=A_{t-1,\tau}^{\mathrm{PPO}}$ from Eq.~\eqref{eq:ppo_advantage_std} for PPO, and $a_{t-1,\tau}=A_{t-1,\tau}^{\mathrm{SDPO}}$ from Eq.~\eqref{eq:sdpo_attribution_std} for SDPO. For group-relative policy optimization, we use GRPO as the representative instantiation and normalize the response-level attribution within each group of size $G$:
\begin{equation*}
a_{t-1,i}^{\mathrm{group}}
=
\frac{G A_{t-1,i}^{\mathrm{group}}}
{\sum_{j=1}^{G}\lvert A_{t-1,j}^{\mathrm{group}}\rvert}.
\end{equation*}
where $A_{t-1,i}^{\mathrm{group}}$ follows the group-relative advantage definition in Eq.~\eqref{eq:grpo_advantage_std}. This normalization controls the PI attribution scale and is supported by our theory and experiments. For GSPO, DAPO, and other GRPO variants, we apply the same normalization to their corresponding group-level attribution scores. Thus, the same closed-loop feedback operates on PPO advantages, group-relative advantages, and dense SDPO token attributions.

\subsection{Optimization Procedure}
\label{subsec:optimization_procedure}

\begin{algorithm}[tb]
  \caption{PIPO Framework}
  \label{alg:pipo}
  \begin{algorithmic}[1]
    \STATE \textbf{Input:} Base objective $\mathcal{J}_{\mathrm{base}}$, attribution rule $a(\cdot)$, window size $K$, learning rates $\alpha_{\mathrm{PI}}, \alpha_{\mathrm{std}}$
    \STATE Initialize memory $\mathcal{M} \leftarrow \emptyset$ and policy $\theta_1$
    \FOR{$t = 1, 2, \dots, T$}
         \STATE Sample $\mathcal{B}_t \sim \pi_{\theta_t}$ and compute $\mu_t$ using Eq.~\eqref{eq:mu_t}
         \STATE Compute local attribution scores $a_t=a(\mathcal{B}_t)$ for the base RL algorithm
         \STATE Set provisional verified policy $\tilde{\theta}_t \leftarrow \theta_t$
         
         \STATE \textit{\# Phase 2: Verification of the previous exploration step}
         \IF{$|\mathcal{M}| \ge K$}
             \STATE Compute $\mu_{\mathrm{his}}$ and $\sigma_{\mathrm{his}}$ from the latest $K$ records in $\mathcal{M}$
             \STATE Compute $\xi_t$ using Eq.~\eqref{eq:xi_def}
             \STATE Retrieve $(\mathcal{B}_{t-1}, a_{t-1})$ from $\mathcal{M}$
             \STATE Construct $\hat{r}^{\mathrm{PI}}_{t,i}$ for $\mathcal{B}_{t-1}$ using Eq.~\eqref{eq:meta_reward_def}
             \STATE $\tilde{\theta}_t \leftarrow \theta_t + \alpha_{\mathrm{PI}} \nabla_\theta \mathcal{J}_{\mathrm{PI}}(\theta_t;\mathcal{B}_{t-1},\hat r^{\mathrm{PI}}_t)$
         \ENDIF
         
         \STATE \textit{\# Phase 1: Exploration with the base RL algorithm}
         \STATE $\theta_{t+1} \leftarrow \tilde{\theta}_t + \alpha_{\mathrm{std}} \nabla_\theta \mathcal{J}_{\mathrm{base}}(\tilde{\theta}_t;\mathcal{B}_t,a_t)$
         
         \STATE Update $\mathcal{M}$ with $(\mu_t, \mathcal{B}_t, a_t)$
    \ENDFOR
    \STATE \textbf{Output:} Optimized policy $\theta_{T+1}$
  \end{algorithmic}
\end{algorithm}

PIPO follows the two-phase process illustrated in Figure~\ref{fig:pipo}. Conceptually, a base RL algorithm first performs \emph{Phase 1: Exploration}, producing a new policy from local rewards, advantages, or distillation targets. The resulting policy is then evaluated in the next iteration, where PIPO performs \emph{Phase 2: Verification} by checking whether this exploration step led to empirical policy improvement. In Algorithm~\ref{alg:pipo}, the current batch $\mathcal{B}_t$ is used to evaluate the previous policy transition and then to compute local attributions for the next exploration step.

\paragraph{Phase 1: Exploration.}
Starting from the verified policy $\tilde{\theta}_t$, PIPO applies the selected base RL algorithm on the current batch to obtain the next policy:
\begin{equation*}
\theta_{t+1}
\leftarrow
\tilde{\theta}_t + \alpha_{\mathrm{std}} \cdot \nabla_\theta \mathcal{J}_{\mathrm{base}}(\tilde{\theta}_t;\mathcal{B}_t),
\end{equation*}
where $\alpha_{\mathrm{std}}$ is the learning rate and $\mathcal{J}_{\mathrm{base}}$ denotes the base policy-learning objective, instantiated by PPO, GRPO, SDPO, or related variants as detailed in Appendix~\ref{app:baselines}. This phase corresponds to the standard exploration step of policy optimization methods.

\paragraph{Phase 2: Verification.} 
The exploration step above is verified at the next iteration. After $\theta_{t+1}$ samples $\mathcal{B}_{t+1}$, the batch mean $\mu_{t+1}$ is compared with the historical anchor to obtain $\xi_{t+1}$. PIPO then constructs PI rewards for the previous batch $\mathcal{B}_t$ and performs a retrospective verification update:

\begin{equation}
\label{eq:meta_update_step}
\begin{gathered}
\tilde{\theta}_{t+1}
\leftarrow
\theta_{t+1} + \alpha_{\mathrm{PI}} \cdot \nabla_\theta \mathcal{J}_{\mathrm{PI}}(\theta_{t+1};\mathcal{B}_{t}),
\\
\mathcal{J}_{\mathrm{PI}}(\theta;\mathcal{B}_{t})
=
\frac{1}{N} \sum_{i=1}^N \min\!\Big( \rho_i^{\mathrm{PI}}(\theta) \hat{r}^{\mathrm{PI}}_{t+1,i},\; \mathrm{clip}(\rho_i^{\mathrm{PI}}(\theta), 1-\epsilon, 1 +\epsilon)\hat{r}^{\mathrm{PI}}_{t+1,i} \Big),
\end{gathered}
\end{equation}
where $\alpha_{\mathrm{PI}}$ is the learning rate, and $\rho_i^{\mathrm{PI}}$ is the retrospective importance ratio against the previous policy $\pi_{\theta_t}$ that generated batch $\mathcal{B}_t$. Its token-level form is

\begin{equation*}
\rho_{t,i,\tau}^{\mathrm{PI}}(\theta)
=
\frac{\pi_{\theta}(y_{t,i,\tau}\mid q_{t,i},y_{t,i,<\tau})}
{\pi_{\theta_t}(y_{t,i,\tau}\mid q_{t,i},y_{t,i,<\tau})},
\end{equation*}
which is evaluated at $\theta=\theta_{t+1}$ in Eq.~\eqref{eq:meta_update_step}. Sequence-level objectives use the analogous sequence ratio.

This retrospective update modulates the previous local attributions according to observed policy improvement: positive feedback reinforces the previous update, while negative feedback scales the corresponding PI signal through $\phi_\lambda(\cdot)$. The resulting verified policy $\tilde{\theta}_{t+1}$ is then used as the starting point for the next base RL update.

\section{Experiments}
\label{sec:experiments}

We design our experiments to answer the following key questions: 
\emph{(i)} Does closed-loop policy improvement yield consistent gains across policy optimization families, model scales, and reasoning benchmarks? 
\emph{(ii)} How does the retrospective policy-improvement mechanism affect training dynamics, particularly regarding optimization stability against data sampling stochasticity and reasoning behavior? 
\emph{(iii)} Does PIPO remain effective when trained and evaluated on code and tool-use RL tasks? 
\emph{(iv)} How sensitive is PIPO to its core design choices, such as the rectification strategy and the historical window size? 
To address these questions, we evaluate PIPO as a plug-and-play module on top of PPO, group-relative, and self-distillation baselines across different model scales, datasets, and ablation settings.

\paragraph{Setup Details.}
We evaluate PIPO across two model scales (Qwen3-4B-Base and Qwen3-8B-Base\cite{yang2025qwen3}) on diverse reasoning benchmarks. For mathematical reasoning, we train on the MATH dataset\cite{NEURIPS_DATASETS_BENCHMARKS2021_be83ab3e} and evaluate on MATH500\cite{NEURIPS_DATASETS_BENCHMARKS2021_be83ab3e}, AIME 2025, AMC 2023, Minerva, and OlympiadBench\cite{he2024olympiadbench}. All reported methods use the same answer extraction and verification pipeline, and our comparisons focus on controlled relative improvements under this evaluator. We further train and evaluate on task-specific code and tool-use datasets, and use SciKnowEval\cite{feng2024sciknoweval} for a focused self-distillation study. More implementation details are provided in Appendix~\ref{app:experimental_details}.

\subsection{Main Results}
\label{subsec:main_results}

\paragraph{Overall Performance.}
As detailed in Table~\ref{tab:main_math_results}, PIPO improves the average performance across policy optimization families and model scales. The gains appear for PPO, group-relative methods, and self-distillation settings, suggesting that policy-improvement feedback is complementary to different forms of local attribution. Rather than replacing the base learning signal, PIPO uses retrospective feedback to modulate it according to observed inter-iteration progress, helping the training process remain aligned with realized task performance.

\paragraph{Training Dynamics.} 

Figure~\ref{fig:main_results} demonstrates the temporal training dynamics of the Qwen3-4B-Base model, illustrating the evolution of average accuracy, critic reward, and response length. Notably, the reported Pass@1 accuracy represents the average performance across all five evaluated mathematical reasoning benchmarks. Throughout the training process, algorithms augmented with PIPO converge to a higher and more stable performance plateau in both average accuracy and critic reward compared to their original baselines. Furthermore, PIPO promotes sustained growth in response length, incentivizing the model to discover deeper, more sophisticated Chain-of-Thought (CoT) reasoning trajectories to tackle challenging tasks.

\begin{table*}[t]
    \centering
    \caption{Main Results on Mathematical Reasoning. We report Pass@1 accuracy (\%) for Qwen3-4B-Base and Qwen3-8B-Base. Improvements over the respective open-loop baselines are highlighted in \textbf{bold}.}
    \label{tab:main_math_results}
    \vspace{-0.2cm}
    \resizebox{\textwidth}{!}{
    \setlength{\tabcolsep}{5pt}
    \begin{tabular}{l | ccccc c || ccccc c}
        \toprule
        & \multicolumn{6}{c||}{\large \textit{\textbf{Qwen3-4B-Base}}} & \multicolumn{6}{c}{\large \textit{\textbf{Qwen3-8B-Base}}} \\
        \textbf{Method} & MATH500 & AIME25 & AMC23 & Minerva & Olympiad & \textbf{AVG} & MATH500 & AIME25 & AMC23 & Minerva & Olympiad & \textbf{AVG} \\
        \midrule
        Base Model & 58.2 & 7.4 & 45.0 & 14.0 & 28.6 & 30.6 & 65.0 & 11.1 & 45.0 & 17.3 & 31.1 & 33.9 \\
        \midrule
        PPO & 80.3 & 11.1 & 62.5 & 25.0 & 41.2 & 44.0 & 81.3 & 18.5 & 67.5 & 25.4 & 43.0 & 47.1 \\
        \textbf{+PIPO} & 80.1 & \textbf{18.5} & \textbf{65.0} & \textbf{25.7} & 39.7 & \textbf{45.8} & \textbf{83.7} & 18.5 & \textbf{70.0} & \textbf{28.3} & \textbf{43.3} & \textbf{48.8} \\
        \midrule
        GRPO & 79.3 & 18.5 & 60.0 & 21.0 & 40.5 & 43.9 & 80.1 & 22.2 & 67.5 & 27.6 & 42.4 & 48.0 \\
        \textbf{+PIPO} & \textbf{80.9} & 18.5 & 60.0 & \textbf{26.8} & \textbf{44.4} & \textbf{46.1} & \textbf{82.3} & 22.2 & 67.5 & \textbf{29.0} & \textbf{43.6} & \textbf{48.9} \\
        \midrule
        GSPO & 78.5 & 18.5 & 62.5 & 23.2 & 39.8 & 44.5 & 81.7 & 22.2 & 67.5 & 26.8 & 45.5 & 48.7 \\
        \textbf{+PIPO} & \textbf{80.3} & \textbf{22.2} & 60.0 & \textbf{24.6} & \textbf{41.8} & \textbf{45.8} & \textbf{83.5} & \textbf{25.9} & \textbf{72.5} & \textbf{28.3} & \textbf{47.5} & \textbf{51.5} \\
        \midrule
        DAPO & 81.3 & 22.2 & 65.0 & 21.7 & 41.8 & 46.4 & 85.3 & 25.9 & 75.0 & 27.9 & 48.7 & 52.6 \\
        \textbf{+PIPO} & \textbf{82.7} & 22.2 & \textbf{70.0} & \textbf{25.0} & \textbf{46.3} & \textbf{49.2} & \textbf{86.3} & \textbf{29.6} & 75.0 & \textbf{30.9} & \textbf{52.2} & \textbf{54.8} \\
        \bottomrule
    \end{tabular}
    }
\end{table*}

\begin{figure*}[t]
  \centering
  \includegraphics[width=\textwidth]{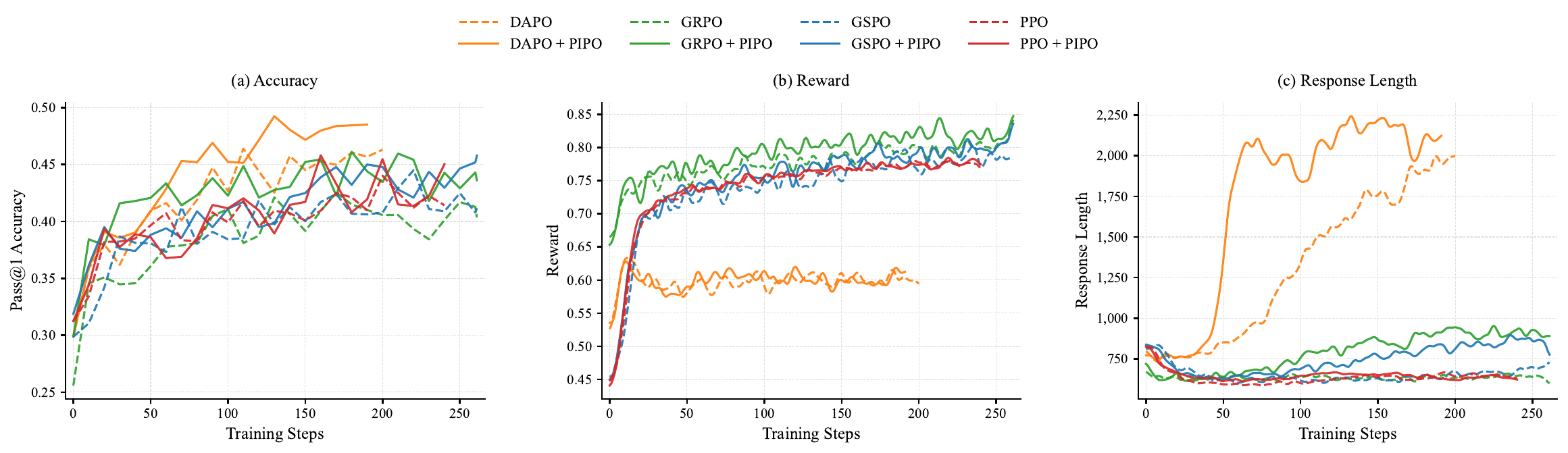}
  \caption{Comparison of training dynamics on Qwen3-4B-Base. (a) Average Pass@1 accuracy evolution across the five evaluated mathematical reasoning benchmarks. (b) Training reward evolution. (c) Response length evolution. Methods augmented with PIPO (solid lines) generally reach higher and more stable trajectories than their respective open-loop baselines (dashed lines).}
  \label{fig:main_results}
\end{figure*}

\subsection{Training on Code and Tool-Use RL Tasks}
\label{subsec:code_tool_results}

\begin{table*}[h]
    \centering
    \caption{Code and Tool-Use RL Results for Qwen3-4B-Base.}
    \label{tab:code_tool_results}
    \setlength{\tabcolsep}{5pt}
    \begin{small}
    \begin{tabular}{lcccc cccc c}
        \toprule
        & \multicolumn{5}{c}{\textbf{Code}} & \multicolumn{4}{c}{\textbf{Tool-Use}} \\
        Method & TACO & LCBv6 & HumanEval & MBPP & Code AVG & RLLA & BFCL & APIBank & Tool AVG \\
        \midrule
        PPO & 23.4 & 31.2 & 85.4 & 41.4 & 45.4 & 87.1 & 66.6 & 90.4 & 81.4 \\
        \textbf{+PIPO} & 22.1 & \textbf{34.9} & \textbf{86.0} & \textbf{45.0} & \textbf{47.0} & \textbf{89.0} & \textbf{71.2} & 90.4 & \textbf{83.5} \\
        \midrule
        GRPO & 23.8 & 31.5 & 83.5 & 44.2 & 45.8 & 85.0 & 58.2 & 90.6 & 77.9 \\
        \textbf{+PIPO} & \textbf{25.4} & \textbf{32.5} & \textbf{86.0} & 43.6 & \textbf{46.9} & \textbf{85.6} & \textbf{61.7} & \textbf{91.7} & \textbf{79.7} \\
        \midrule
        GSPO & 23.3 & 32.6 & 81.7 & 45.0 & 45.7 & 85.6 & 59.7 & 89.6 & 78.3 \\
        \textbf{+PIPO} & \textbf{23.4} & 32.1 & \textbf{84.2} & \textbf{45.8} & \textbf{46.4} & 84.0 & \textbf{62.8} & \textbf{90.6} & \textbf{79.1} \\
        \midrule
        DAPO & 25.4 & 33.5 & 86.0 & 43.8 & 47.2 & 87.4 & 62.2 & 90.2 & 79.9 \\
        \textbf{+PIPO} & \textbf{27.3} & \textbf{36.4} & \textbf{87.2} & \textbf{44.4} & \textbf{48.8} & \textbf{87.7} & \textbf{67.7} & \textbf{90.4} & \textbf{81.9} \\
        \bottomrule
    \end{tabular}
    \end{small}
\end{table*}

To evaluate PIPO beyond mathematical reasoning, we train and evaluate on task-specific code and tool-use RL tasks, where correctness can be verified by executable tests or task-level tool feedback. Table~\ref{tab:code_tool_results} reports the comparison across PPO, GRPO, GSPO, and DAPO, each paired with its PIPO-augmented counterpart.

Across both task families, PIPO improves the average score for all four base algorithms. The gains are especially clear on aggregate code performance and on BFCL-style tool-use evaluation, where noisy execution feedback and heterogeneous task difficulty can make local update signals unreliable. These results suggest that retrospective verification is useful beyond mathematical reasoning whenever task-level feedback is verifiable.

\subsection{PIPO with Self-Distillation}
\label{subsec:self_distillation_results}

\begin{table}[h]
    \centering
    \caption{Self-Distillation Study on SciKnowEval.}
    \label{tab:sdpo_sciknow}
    \setlength{\tabcolsep}{7pt}
    \begin{tabular}{lccccc}
        \toprule
        Method & Biology & Chemistry & Material & Physics & AVG \\
        \midrule
        Qwen3-8B & 4.0 & 13.3 & 39.4 & 30.0 & 21.7 \\
        \midrule
        GRPO & 40.0 & 72.9 & 71.3 & 47.5 & 57.9 \\
        \textbf{+PIPO} & \textbf{54.0} & 60.5 & \textbf{73.4} & \textbf{53.8} & \textbf{60.4} \\
        SDPO & 52.0 & 76.7 & 73.4 & 77.5 & 69.9 \\
        \textbf{+PIPO} & \textbf{60.0} & \textbf{79.1} & \textbf{79.8} & 76.3 & \textbf{73.8} \\
        \bottomrule
    \end{tabular}
\end{table}

We also conduct a focused study on whether policy improvement feedback complements self-distillation. Following the SciKnowEval configuration used in SDPO~\cite{hubotter2026reinforcement}, we compare GRPO and SDPO with their PIPO-augmented counterparts on Qwen3-8B. As shown in Table~\ref{tab:sdpo_sciknow}, PIPO improves both GRPO and SDPO on average. This suggests that PIPO remains useful even when the base algorithm already receives dense token-level feedback from a privileged self-teacher, since it verifies whether the resulting update translates into measurable downstream progress.

\subsection{Ablation Studies}
\label{sec:ablation}

\begin{table}[htbp]
\centering
\caption{Ablation on Rectification Coefficient $\lambda$ (Qwen3-4B-Base).}
\label{tab:ablation_lambda}
\begin{tabular}{lcccccc}
\toprule
$\lambda$ & MATH500 & AIME25 & AMC23 & Minerva & Olympiad & AVG \\
\midrule
GRPO Baseline & 79.3 & 18.5 & 60.0 & 21.0 & 40.5 & 43.9 \\
0 & 80.5 & 18.5 & 62.5 & 22.1 & 39.7 & 44.7 \\
0.05 & 78.5 & \textbf{22.2} & 62.5 & 26.1 & 39.1 & 45.7 \\
0.1 & 80.9 & 18.5 & 60.0 & \textbf{26.8} & \textbf{44.4} & \textbf{46.1} \\
0.2 & \textbf{81.9} & 18.5 & 60.0 & 25.0 & 42.6 & 45.6 \\
0.5 & 81.3 & 18.5 & \textbf{65.0} & 22.8 & 40.9 & 45.7 \\
1 & 78.3 & \textbf{22.2} & 57.5 & 23.5 & 41.2 & 44.5 \\
\bottomrule
\end{tabular}
\end{table}

\textbf{Rectification Coefficient $\lambda$.} A key component of our closed-loop mechanism is the rectification function $\phi_{\lambda}(\cdot)$, where $\lambda$ controls the suppression strength applied to negative improvement signals $(\xi_{t} < 0)$. Table \ref{tab:ablation_lambda} compares different suppression intensities. While all evaluated $\lambda$ configurations successfully outperform the baseline, we observe that hard suppression ($\lambda=1$) is overly conservative. Applying a strong penalty effectively negates exploratory updates, pulling the policy back toward its previous state and heavily restricting the model's ability to discover better reasoning trajectories. Conversely, a reinforce-only setting that completely removes the suppression ($\lambda=0$) still improves over the baseline by tolerating normal exploration without any penalty. Ultimately, a properly calibrated soft suppression ($\lambda=0.1$) achieves the optimal balance (46.1\% average accuracy). This demonstrates that gently dampening harmful updates acts as a stabilizing filter, reducing the impact of genuine performance drops without overly limiting the necessary exploration space.

\textbf{Window Size $K$.} We also conduct an ablation study on the sliding window size $K$, which acts on the smoothed historical baseline. Overall, we find that a moderate window size ($K=8$) serves as a robust equilibrium, balancing the accuracy of the historical capability estimation with the responsiveness needed to track rapid capability gains. Further analyses on stability across multiple random seeds and wall-clock efficiency are detailed in Appendix~\ref{sec:additional_experiments}.




\section{Theoretical Analysis}
\label{sec:theoretical_analysis}

We provide theoretical evidence and an illustrative group-relative case study for PIPO. Building on the objective-alignment result in Theorem~\ref{prop:pirl_alignment}, we first show that PIPO yields local ascent toward the PIRL improvement objective under sign-consistent feedback. We then analyze group-relative optimization as a stylized case where local batch statistics can exhibit boundary sensitivity.

\subsection{Approximate PIRL Ascent}
\label{subsec:general_modulation}

PIPO only requires the base RL algorithm to provide a local attribution score that indicates the direction and relative magnitude of a policy update. For PPO, the attribution is the original GAE advantage; for group-relative methods, it is the normalized group-relative advantage; for SDPO, it is the self-teacher log-ratio attribution. In all cases, the local signal determines \emph{where} to assign credit, while $\phi_\lambda(\xi_t)$ determines whether empirical improvement supports the previous update direction.

Let $g_{t-1}^{\mathrm{base}}$ denote the first-order update direction induced by the base RL algorithm on $\mathcal{B}_{t-1}$. Since the PI reward factorizes into local attribution and scalar improvement feedback, the retrospective PIPO step can be locally written as
\begin{equation*}
\Delta \theta_{\mathrm{PI}}
=
c_t \cdot g_{t-1}^{\mathrm{base}}
+ \mathcal{O}(\|\theta_t-\theta_{t-1}\|),
\qquad
c_t \propto \phi_\lambda(\xi_t).
\end{equation*}

\begin{proposition}[Approximate PIRL Ascent under Sign-Consistent Feedback]
\label{prop:approx_pirl_ascent}
Let $\mathcal{H}_{t-1}$ denote the training history up to iteration $t-1$. For each historical learning unit $u_{t-1,i}$ with local attribution $a_{t-1,i}$, the PI reward satisfies:
\begin{enumerate}[label=(\alph*)]
\item Conditioned on $\mathcal{H}_{t-1}$ and $\mathcal{B}_{t-1}$,
\[
\mathbb{E}\!\left[\hat r^{\mathrm{PI}}_{t,i}\mid \mathcal{H}_{t-1},\mathcal{B}_{t-1}\right]
=
a_{t-1,i}\,
\mathbb{E}\!\left[\phi_\lambda(\xi_t)\mid \mathcal{H}_{t-1},\mathcal{B}_{t-1}\right].
\]
\item Under the sign-consistency condition formalized in Appendix~\ref{sec:proof_approx_pirl_ascent}, a local bounded-gradient condition, and the above local trust-region approximation, the retrospective PIPO update is locally aligned with the first-order improvement direction of the PIRL objective:
\[
\left\langle \nabla J_{\mathrm{RL}}(\theta_t), \Delta\theta_{\mathrm{PI}}\right\rangle
\ge
-\mathcal{O}(\|\theta_t-\theta_{t-1}\|).
\]
\end{enumerate}
\end{proposition}

This gives a local bridge from the ideal PIRL objective to the practical PIPO update: when empirical improvement is informative about update quality, PIPO reinforces directions aligned with realized progress and suppresses directions associated with performance decline.

\subsection{Illustrative Case Study: Group-Relative Boundary Sensitivity}
\label{subsec:geometric_rectification}

Group-relative optimization provides an analyzable example of how an open-loop local surrogate can become sensitive in sparse-reward regimes. The following result is not required for applying PIPO to PPO or self-distillation; rather, it illustrates one concrete boundary regime where closed-loop verification can be useful.

\begin{figure}[t]
    \centering
    \includegraphics[width=\linewidth]{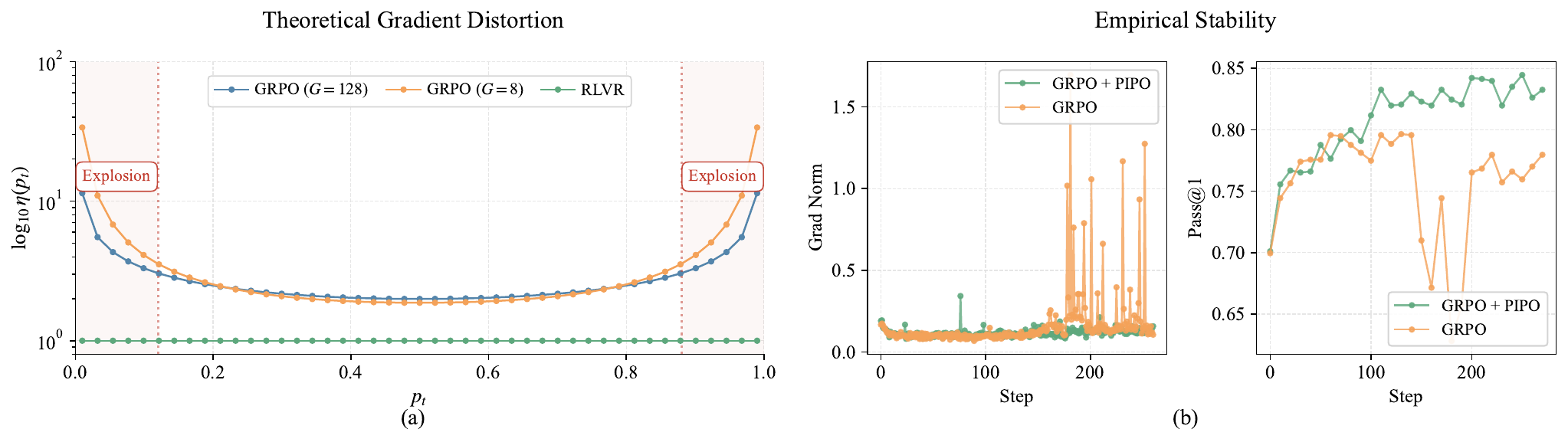}
    \caption{ \textbf{Boundary sensitivity and empirical instability of group-relative optimization.} 
    (a) \textbf{Scaling Sensitivity:} The gradient scaling factor $\eta(p_t)$ evaluated across success rates $p_t$. As established in Corollary~\ref{cor:boundary_sensitivity}, GRPO ($G=8, 128$) can exhibit severe sensitivity near the boundaries ($p_t \to 0, 1$). 
    (b) \textbf{Empirical Stability:} Standard GRPO suffers from gradient norm spikes (left) and Pass@1 degradation (right), while incorporating PIPO stabilizes training.}
    \label{fig:comp}
\end{figure}

For a query $q$, let $S_q := \sum_{i=1}^G R(q,y_i)$ denote the number of successes in a group of size $G$. Degenerate batches ($S_q \in \{0,G\}$) produce identical rewards across the group, resulting in zero group-relative advantages. Consequently, GRPO optimization is conditioned on the non-degenerate event $\mathcal{E} := \{1 \le S_q \le G-1\}$. Under a local trust-region approximation, the reliance on group-relative advantage induces a state-dependent scaling of the ideal RL gradient.

\begin{theorem}[Group-Relative Scaling Sensitivity]
\label{the:grpo_misalignment}
Under the above conditions, the expected GRPO update direction satisfies
\[
g_{\mathrm{GRPO}}(\theta)
=
\eta(p(q;\theta)) \cdot g_{\mathrm{ideal}},
\]
where $g_{\mathrm{ideal}} = \nabla_\theta J_{\mathrm{RL}}(\theta)$ and the scaling factor is
\[
\eta(p)
=
\frac{\sum_{k=1}^{G-1}\sqrt{k(G-k)}\binom{G}{k}p^k (1-p)^{G-k}}{G\,p(1-p)\left(1 - p^G - (1-p)^G\right)}.
\]
Thus, GRPO applies a state-dependent rescaling to the ideal gradient under this simplified analysis.
\end{theorem}

\begin{corollary}[Boundary Scaling Sensitivity]
\label{cor:boundary_sensitivity}
At the probability boundaries ($p \to 0$ or $p \to 1$), the gradient scaling factor exhibits a symmetric singularity:
\[
\eta(p) \sim \frac{\sqrt{G-1}}{G \cdot p(1 - p)}.
\]
\end{corollary}

Theorem~\ref{the:grpo_misalignment} and Corollary~\ref{cor:boundary_sensitivity} show that, in a stylized sparse-reward setting, group-relative normalization can amplify rare non-degenerate signals for extremely hard or easy queries. This identifies a boundary-sensitive component already present in the core group-relative estimator, even before considering more complex practical mechanisms such as clipping, filtering, or gradient norm control. PIPO addresses this issue at a different level: rather than modifying the group estimator itself, it checks whether the resulting update improves the subsequent policy and modulates the update accordingly.

\begin{proposition}[Boundary Sensitivity Damping]
\label{prop:rectification}
Under bounded local updates, if the retrospective modulation uses the realized policy improvement $\Delta J_t$ as ideal feedback, PIPO can damp the leading boundary scaling factor in the group-relative case up to first-order approximation. In particular, the retrospective PI gradient remains a bounded modulation of the ideal success-probability gradient near the boundaries:
\begin{equation*}
    \mathbb{E}[\nabla_\theta \mathcal{J}_{\mathrm{PI}}]
    =
    \kappa(p)\nabla_\theta p + \mathcal{O}(\alpha_{\mathrm{std}}),
    \qquad
    \kappa(p)=\mathcal{O}(1),
    \quad \text{as } p \to 0 \text{ or } 1.
\end{equation*}
\end{proposition}

Proposition~\ref{prop:rectification} formalizes the damping role of idealized policy-improvement feedback in this stylized group-relative setting. In practice, PIPO replaces $\Delta J_t$ with the standardized empirical feedback $\phi_\lambda(\xi_t)$. The key point is not that PIPO guarantees global improvement, but that policy-improvement feedback can reduce the leading-order boundary sensitivity of a local group-relative update.


\section{Related Work}
\label{sec:related_work}

\paragraph{RL Post-Training with Verifiable and Rich Feedback.}
RL post-training has rapidly become a core approach for improving LLM reasoning and agentic capabilities.
One line of work relies on verifiable rewards from answer checkers, program execution, or tool feedback, enabling scalable optimization for mathematical reasoning, code generation, and tool-use tasks~\cite{yang2025qwen3,Guo_2025,zou2025transformer,chen2026scoring,xie2026uniarm,chen2026weakdriven,wu2025ucoder}.
Another line studies reinforcement learning with rich feedback, where execution traces, judge comments, or privileged information provide denser training signals.
On-policy distillation and self-distillation methods further transform such feedback into teacher-guided learning targets~\cite{song2026survey,li2026rethinking,hubotter2026reinforcement,zhao2026self}.
These works mainly differ in how they construct local learning signals.
PIPO is complementary: it does not introduce a new verifier, critic, or teacher, but adds cross-iteration policy-improvement feedback on top of existing RL post-training objectives.

\paragraph{Critic-Free RL and Advantage Estimation.}
Following the success of GRPO\cite{shao2024deepseekmath} and RLOO\cite{ahmadian2024rloo} in eliminating value networks within PPO\cite{schulman2017proximal}, the field has seen a surge of methodological refinements to stabilize group-relative estimation\cite{liu2026cdrrm,lu2026contextual,li2026adaptive,dai2026harder,plyusov2026fgrpo,qi2026rethinkingtrustregionllm,li2026optimaltokenbaselinevariance,qiyuan2026hipo}. 
For example, HA-DW\cite{yang2026your} identifies a fundamental bias in group-relative advantage estimation, showing that it systematically underestimates hard cases and overestimates easy questions.
Dr.GRPO~\cite{liu2025understanding} removes variance normalization, GMPO~\cite{zhao2025gmpo} adopts geometric averaging to handle outliers, while DCPO~\cite{yang2025dcpo} and CISPO~\cite{minimax2025cispo} introduce adaptive clipping and smooth standardization to alleviate vanishing gradients.
Other approaches address estimation bias via reweighting or baseline design, such as OPO~\cite{hao2025policy}.
Finer-grained credit assignment has also been explored by decomposing rewards at the sub-trajectory level~\cite{guo2025spo,sun2025ktae}.
Despite these advances, most methods operate within isolated batches and overlook temporal consistency across policy updates.
PIPO departs from this paradigm by introducing \textit{retrospective verification}, which explicitly incorporates cross-iteration policy improvement signals.

\paragraph{Stability in RL Post-Training.}
Stability has been widely studied in RL post-training, including exploration regularization~\cite{hao2025policy,cheng2025entropy,deng2025decomposing,song2025curiosity}, larger rollout budgets~\cite{hu2025rollouts,he2026random}, curriculum-based training schedules~\cite{chen2025curriculum,setlur2025prm,sundaram2026soar}, and objectives that preserve generative diversity~\cite{tajwar2026maximum,wu2026distribution,yue2025impact,gxchen2025collapse,chen2026passk,chen2026exploration,chen2026eepo}.
These methods primarily adjust the local training objective or sampling process.
PIPO instead introduces an inter-iteration verification signal, using historical performance feedback to modulate policy updates after their empirical effect is observed.

\paragraph{Beyond Intra-Batch Estimation.}
Optimizing policies based on isolated batches often suffers from high-variance estimates due to the lack of temporal or global context during the gradient computation\cite{shi2026experientialreinforcementlearning,jiang2026overthinking,huang2026on,huang2026pros,liang2026beyond,xie2026sample}.
In representation learning, this issue has been addressed using cross-batch memory mechanisms to approximate global statistics~\cite{wang2020xbm,wang2021cbns}.
Analogously, recent RL post-training methods have begun incorporating historical information.
SPO~\cite{xu2025single} employs a moving-average baseline to avoid degenerate groups, while HA-DW~\cite{yang2026your} and OPO~\cite{hao2025policy} leverage historical statistics to correct bias and reduce variance.
However, these approaches primarily use history for baseline adjustment or reweighting.
PIPO extends this line of work by treating historical performance as a dynamic performance anchor, enabling retrospective verification that explicitly suppresses policy updates inconsistent with sustained improvement.

\section{Conclusion}

This work addresses a fundamental limitation of existing RL post-training methods: optimization relies on step-wise reward, advantage, or distillation signals without explicitly verifying whether an update has produced a better policy over time. We introduce PIRL, which reframes post-training as the maximization of inter-iteration performance gains, emphasizing verifiable progress rather than instantaneous surrogate objectives. Building on PIRL, we propose PIPO, a closed-loop algorithm that operationalizes this principle through retrospective verification. PIPO uses empirical improvement signals to gate and rectify local attribution signals from PPO, group-relative, and self-distillation objectives. We show theoretically that PIPO locally aligns policy updates with the PIRL improvement objective and analyze group-relative optimization as a concrete case of open-loop distortion.
Empirically, PIPO yields consistent gains across policy optimization families on mathematical reasoning, code, tool-use, and self-distillation settings.

\section{Impact Statement}

This work introduces Policy Improvement Reinforcement Learning as a conceptual shift in RL post-training. Rather than treating optimization as a sequence of independent, step-wise updates driven by instantaneous surrogate signals, PIRL reframes learning as the accumulation of verifiable inter-iteration improvements. This perspective highlights a missing temporal dimension in existing RL methods and provides a principled lens for analyzing stability, bias, and performance decline phenomena that arise in sparse-feedback reasoning tasks. By making policy improvement itself an explicit optimization target, PIRL offers a unifying objective that remains fully aligned with final-task performance while enabling new forms of feedback control. The proposed PIPO algorithm demonstrates how this principle can be instantiated in practice through retrospective verification, without introducing additional critics or heavy heuristics. More broadly, PIRL suggests a general design paradigm for reinforcement learning systems in which progress, not just reward, becomes the central unit of optimization, with potential implications for scalable, stable post-training of large language models beyond the specific algorithms studied here.

\bibliographystyle{plain}
\bibliography{references}

\newpage
\appendix

\section{Training and Evaluation Details}
\label{app:experimental_details}

\subsection{Models and Datasets}
\paragraph{Models}
For the main RL post-training experiments, we evaluate all methods on Qwen3-4B-Base and Qwen3-8B-Base~\cite{yang2025qwen3}, covering two representative model scales.
We deliberately use base (non-instruction-tuned) models to assess the ability of RL methods to elicit reasoning behaviors from scratch, without relying on supervised or alignment priors.
The self-distillation study is the exception: following the SDPO setting~\cite{hubotter2026reinforcement}, we use Qwen3-8B with stronger instruction-following ability, since constructing a reliable privileged self-teacher from a base model is unstable.

\paragraph{Datasets}
Mathematical reasoning experiments are trained on the MATH dataset~\cite{NEURIPS_DATASETS_BENCHMARKS2021_be83ab3e}, which contains approximately 7.5k multi-step mathematical reasoning problems.
Evaluation is performed on MATH500, a representative subset of the MATH test set; AIME 2025 and AMC 2023, which measure competition-level mathematical reasoning; Minerva~\cite{lewkowycz2022solving}, which emphasizes scientific and quantitative reasoning; and OlympiadBench~\cite{he2024olympiadbench}, which contains Olympiad-level mathematical and scientific problems.
All methods use the same answer extraction and verification pipeline. Absolute scores may differ from prior reports due to differences in prompts, decoding, and normalization, while comparisons in this paper are made under the same evaluator.
For code RL, models are trained on TACO~\cite{li2023taco} and evaluated on TACO, LiveCodeBench v6 (LCBv6)~\cite{jain2024livecodebench}, HumanEval~\cite{chen2021evaluating}, and MBPP~\cite{austin2021program} using executable tests or task-specific correctness checks. For tool-use RL, models are trained on ToolRL/RLLA~\cite{qian2025toolrl} and evaluated on RLLA~\cite{qian2025toolrl}, BFCL~\cite{patil2025bfcl}, and APIBank~\cite{li2023apibank}, where policies must invoke tools or external functions and receive verifiable task-level feedback. In addition, we use SciKnowEval\cite{feng2024sciknoweval} for the focused self-distillation study in Section~\ref{subsec:self_distillation_results}. This benchmark evaluates large language models across biology, chemistry, physics, and materials science.

\subsection{Baselines}
\label{app:baselines}
We instantiate PIPO on top of PPO~\cite{schulman2017proximal}, GRPO~\cite{shao2024deepseekmath}, GSPO~\cite{zheng2025group}, DAPO~\cite{yu2025dapo}, and SDPO~\cite{hubotter2026reinforcement}. Let $x_t$ denote a prompt at training step $t$, and let $o_{t,i}=(y_{t,i,1},\ldots,y_{t,i,|o_{t,i}|})$ be the $i$-th sampled response with verifier reward $R_{t,i}$. For token-level objectives, we define the importance ratio
\begin{equation*}
\rho_{t,i,\tau}(\theta)
=
\frac{\pi_\theta(y_{t,i,\tau}\mid x_t,o_{t,i,<\tau})}
{\pi_{\theta_{\mathrm{old}}}(y_{t,i,\tau}\mid x_t,o_{t,i,<\tau})}.
\end{equation*}

PPO uses a learned value function and GAE advantage $A^{\mathrm{PPO}}_{t,\tau}$; since PPO uses a single rollout per prompt, we omit the response index and write $r_{t,\tau}$ for the token-level ratio. Its clipped surrogate is
\begin{equation*}
\mathcal{J}_{\mathrm{PPO}}(\theta)
=
\mathbb{E}\!\left[
\frac{1}{|o_t|}\sum_{\tau=1}^{|o_t|}
\min\!\left(
\rho_{t,\tau}(\theta)A^{\mathrm{PPO}}_{t,\tau},
\mathrm{clip}(\rho_{t,\tau}(\theta),1-\epsilon,1+\epsilon)A^{\mathrm{PPO}}_{t,\tau}
\right)
\right].
\end{equation*}
GRPO removes the value network by assigning each response a group-normalized advantage,
\begin{equation*}
A^{\mathrm{GRPO}}_{t,i}
=
\frac{R_{t,i}-\mathrm{mean}(\{R_{t,j}\}_{j=1}^{G})}
{\mathrm{std}(\{R_{t,j}\}_{j=1}^{G})},
\end{equation*}
and broadcasts it to all tokens in the response:
\begin{equation*}
\mathcal{J}_{\mathrm{GRPO}}(\theta)
=
\mathbb{E}\!\left[
\frac{1}{G}\sum_{i=1}^{G}
\frac{1}{|o_{t,i}|}\sum_{\tau=1}^{|o_{t,i}|}
\min\!\left(
\rho_{t,i,\tau}(\theta)A^{\mathrm{GRPO}}_{t,i},
\mathrm{clip}(\rho_{t,i,\tau}(\theta),1-\epsilon,1+\epsilon)A^{\mathrm{GRPO}}_{t,i}
\right)
\right].
\end{equation*}
GSPO keeps the group-relative advantage but replaces the token-wise ratio with a sequence-level ratio
\begin{equation*}
\rho^{\mathrm{seq}}_{t,i}(\theta)
=
\exp\!\left(
\frac{1}{|o_{t,i}|}\sum_{\tau=1}^{|o_{t,i}|}
\log \rho_{t,i,\tau}(\theta)
\right),
\end{equation*}
yielding
\begin{equation*}
\mathcal{J}_{\mathrm{GSPO}}(\theta)
=
\mathbb{E}\!\left[
\frac{1}{G}\sum_{i=1}^{G}
\min\!\left(
\rho^{\mathrm{seq}}_{t,i}(\theta)A_{t,i},
\mathrm{clip}(\rho^{\mathrm{seq}}_{t,i}(\theta),1-\epsilon_{\mathrm{low}},1+\epsilon_{\mathrm{high}})A_{t,i}
\right)
\right].
\end{equation*}
DAPO further uses asymmetric clipping and filters degenerate groups. Denoting the retained prompt groups by $\mathcal{F}_t$, its objective is
\begin{equation*}
\mathcal{J}_{\mathrm{DAPO}}(\theta)
=
\mathbb{E}_{x_t\in\mathcal{F}_t}\!\left[
\frac{1}{\sum_i |o_{t,i}|}
\sum_{i=1}^{G}\sum_{\tau=1}^{|o_{t,i}|}
\min\!\left(
\rho_{t,i,\tau}(\theta)A_{t,i},
\mathrm{clip}(\rho_{t,i,\tau}(\theta),1-\epsilon_{\mathrm{low}},1+\epsilon_{\mathrm{high}})A_{t,i}
\right)
\right].
\end{equation*}
For self-distillation, SDPO uses privileged information $f$ available after the rollout, such as execution feedback, a verified solution, or the ground-truth option, to construct a feedback-conditioned self-teacher distribution $q_{\theta}(\cdot\mid x_t,f,o_{t,<\tau})$. It then minimizes a token-level reverse KL objective:
\begin{equation*}
\mathcal{L}_{\mathrm{SDPO}}(\theta)
=
\mathbb{E}\!\left[
\sum_{\tau}
D_{\mathrm{KL}}
\left(
\pi_{\theta}(\cdot\mid x_t,o_{t,<\tau})
\;\|\;
\mathrm{stopgrad}\!\left(q_{\theta}(\cdot\mid x_t,f,o_{t,<\tau})\right)
\right)
\right],
\end{equation*}
which provides dense token-level attribution through the self-teacher/student log-probability difference
$A^{\mathrm{SDPO}}_{t,i,\tau}=\log q_{\theta}(y_{t,i,\tau}\mid x_t,f,o_{t,i,<\tau})-\log \pi_{\theta}(y_{t,i,\tau}\mid x_t,o_{t,i,<\tau})$.

PIPO preserves each baseline objective and appends a retrospective policy-improvement update. Let $\mathcal{J}_{\mathrm{base}}$ denote the maximization form of the corresponding baseline objective above, and let $a_{t,i,\tau}$ denote its local attribution score: $A^{\mathrm{PPO}}_{t,\tau}$ for PPO, group-relative advantages for GRPO, GSPO, and DAPO, and the SDPO token attribution for self-distillation. Given the improvement signal $\xi_{t+1}$, PIPO forms
\begin{equation*}
\hat r^{\mathrm{PI}}_{t+1,i,\tau}
=
a_{t,i,\tau}\,\phi_{\lambda}(\xi_{t+1}),
\end{equation*}
and performs the additional clipped update on the previous batch:
\begin{align*}
\theta_{t+1}^{\mathrm{base}}
&=
\theta_t+\alpha_{\mathrm{std}}
\nabla_\theta \mathcal{J}_{\mathrm{base}}(\theta_t;\mathcal{B}_t), \\
\theta_{t+1}^{\mathrm{PIPO}}
&=
\theta_{t+1}^{\mathrm{base}}
+\alpha_{\mathrm{PI}}
\nabla_\theta \mathcal{J}_{\mathrm{PI}}(\theta_{t+1}^{\mathrm{base}};\mathcal{B}_t),
\end{align*}
where
\begin{equation*}
\mathcal{J}_{\mathrm{PI}}(\theta;\mathcal{B}_t)
=
\frac{1}{Z_t}\sum_{i,\tau}
\min\!\left(
\rho_{i,\tau}(\theta)\hat r^{\mathrm{PI}}_{t+1,i,\tau},
\mathrm{clip}(\rho_{i,\tau}(\theta),1-\epsilon,1+\epsilon)\hat r^{\mathrm{PI}}_{t+1,i,\tau}
\right),
\end{equation*}
$Z_t$ is the number of historical learning units, and $\rho_{i,\tau}(\theta)$ is the importance ratio between the retrospective policy and the policy that generated $\mathcal{B}_t$. This formulation makes PIPO a direct plug-in: it leaves the baseline objective intact and only modulates whether the previous update should be reinforced or suppressed according to verified policy improvement.

\subsection{Training Hyperparameters}

\begin{table*}[ht]
\centering
\caption{Hyperparameter settings for PPO, group-relative methods (GRPO, GSPO, DAPO), and their PIPO-enhanced versions}
\label{tab:hyperparameters}
\begin{small}
\resizebox{\textwidth}{!}{
\begin{tabular}{l|l|cccccccc}
\toprule
\multicolumn{2}{c|}{\textbf{Hyperparameter}} & \textbf{PPO} & \textbf{PPO+PIPO} & \textbf{GRPO} & \textbf{GRPO+PIPO} & \textbf{GSPO} & \textbf{GSPO+PIPO} & \textbf{DAPO} & \textbf{DAPO+PIPO} \\
\midrule
\multirow{2}{*}{Compute} 
 & Nodes & 1 & 1 & 1 & 1 & 1 & 1 & 1 & 1 \\
 & GPUs per Node & 8 & 8 & 8 & 8 & 8 & 8 & 8 & 8 \\
\midrule
\multirow{4}{*}{General} 
 & Use KL in Reward & True & True & False & False & False & False & False & False \\
 & Use KL Loss & False & False & True & True & True & True & True & True \\
 & KL Coef & 0.001 & 0.001 & 0.001 & 0.001 & 0.001 & 0.001 & 0.001 & 0.001 \\
 & Optimizer & AdamW & AdamW & AdamW & AdamW & AdamW & AdamW & AdamW & AdamW \\
\midrule
\multirow{5}{*}{Training} 
 & Epochs & 24 & 24 & 3 & 3 & 3 & 3 & 9 & 9 \\
 & Batch Size & 1024 & 1024 & 128 & 128 & 128 & 128 & 128 & 128 \\
 & Mini Batch Size & 512 & 512 & 64 & 64 & 64 & 64 & 64 & 64 \\
 & Micro Batch Size & 4 & 4 & 4 & 4 & 4 & 4 & 4 & 4 \\
 & Learning Rate & $1 \times 10^{-6}$ & $1 \times 10^{-6}$ & $1 \times 10^{-6}$ & $1 \times 10^{-6}$ & $1 \times 10^{-6}$ & $1 \times 10^{-6}$ & $1 \times 10^{-6}$ & $1 \times 10^{-6}$ \\
\midrule
\multirow{2}{*}{Clipping} 
 & Clip High ($\epsilon_{\text{high}}$) & 0.2 & 0.2 & 0.2 & 0.2 & 0.0004 & 0.0004 & 0.28 & 0.28 \\
 & Clip Low ($\epsilon_{\text{low}}$) & 0.2 & 0.2 & 0.2 & 0.2 & 0.0003 & 0.0003 & 0.2 & 0.2 \\
\midrule
\multirow{5}{*}{Rollout} 
 & Rollout ($G$) & 1 & 1 & 8 & 8 & 8 & 8 & 8 & 8 \\
 & Temperature & 1.0 & 1.0 & 1.0 & 1.0 & 1.0 & 1.0 & 1.0 & 1.0 \\
 & Max Prompt Len & 512 & 512 & 512 & 512 & 512 & 512 & 512 & 512 \\
 & Max Response Len & 4096 & 4096 & 4096 & 4096 & 4096 & 4096 & 4096 & 4096 \\
 & Filter Groups & No & No & No & No & No & No & Yes & Yes \\
\midrule
\multirow{3}{*}{\textbf{PIPO Specific}} 
 & Window Size ($K$) & - & 8 & - & 8 & - & 8 & - & 8 \\
 & PI LR ($\alpha_{\text{PI}}$) & - & $1 \times 10^{-6}$ & - & $1 \times 10^{-6}$ & - & $1 \times 10^{-6}$ & - & $1 \times 10^{-6}$ \\
 & Rectification Coefficient($\phi_\lambda(\cdot)$) & - & 0.1 & - & 0.1 & - & 0.1 & - & 0.1 \\
\bottomrule
\end{tabular}
}
\end{small}
\end{table*}

We conduct all RL training within the \textbf{VeRL} framework~\cite{sheng2025hybridflow} on a single compute node equipped with \textbf{8 $\times$ NVIDIA H20 (141GB) GPUs}.
All experiments share the same core hyperparameters (e.g., learning rate and KL coefficient) to ensure fair comparisons.
For PPO, we use batch size 1024 and 24 epochs to match the response-level budget of group-relative methods with batch size 128 and $G=8$.
Method-specific settings, such as clipping strategies, follow the original implementations.
PIPO introduces additional parameters for retrospective verification, including the historical window size $K$ and the policy improvement learning rate.
Detailed configurations for all methods are reported in Table~\ref{tab:hyperparameters}.

\section{Additional Experiments}
\label{sec:additional_experiments}

\subsection{Robust Multi-Sample Reasoning}
\label{sec:passk_results}

To further evaluate reasoning robustness, we measure Pass@8 accuracy on Qwen3-4B-Base. Table~\ref{tab:pass8_results} compares PPO and group-relative baselines under the multi-sample evaluation setting.

\begin{table}[htbp]
    \centering
    \caption{Pass@8 Results for Qwen3-4B-Base.}
    \label{tab:pass8_results}
    \begin{tabular}{lcccccc}
        \toprule
        Method & MATH500 & AIME25 & AMC23 & Minerva & Olympiad & AVG \\
        \midrule
        PPO & 87.2 & 20.4 & 79.0 & 32.9 & 54.0 & 54.7 \\
        \textbf{+PIPO} & \textbf{88.0} & \textbf{28.1} & 72.6 & \textbf{33.0} & 53.9 & \textbf{55.1} \\
        \midrule
        GRPO & 88.3 & 27.6 & 80.2 & 34.7 & 54.1 & 57.0 \\
        \textbf{+PIPO} & \textbf{89.4} & \textbf{30.3} & \textbf{84.0} & \textbf{35.9} & \textbf{57.1} & \textbf{59.3} \\
        \midrule
        GSPO & 87.2 & 29.1 & 81.3 & 34.8 & 57.2 & 57.9 \\
        \textbf{+PIPO} & \textbf{88.9} & \textbf{31.2} & \textbf{82.7} & \textbf{35.9} & \textbf{57.6} & \textbf{59.3} \\
        \midrule
        DAPO & 88.2 & 30.7 & 84.9 & 35.9 & 59.4 & 59.8 \\
        \textbf{+PIPO} & \textbf{89.2} & \textbf{31.4} & \textbf{86.2} & 35.8 & \textbf{60.0} & \textbf{60.5} \\
        \bottomrule
    \end{tabular}
\end{table}

PIPO improves the average Pass@8 score for all evaluated base algorithms, indicating that the closed-loop verification signal remains beneficial under multi-sample decoding. The gains are particularly pronounced for group-relative baselines, suggesting that suppressing performance-degrading updates improves not only single-sample accuracy but also the quality of the candidate solution set.

\subsection{Ablation on Window Size K} 

\begin{table}[htbp]
\centering
\caption{Ablation on Window Size $K$ on Qwen3-4B-Base.}
\label{tab:ablation_k}
\begin{tabular}{lcccccc}
\toprule
$K$ & MATH500 & AIME25 & AMC23 & Minerva & Olympiad & AVG \\
\midrule
GRPO Baseline & 79.3 & 18.5 & 60.0 & 21.0 & 40.5 & 43.9 \\
2 & 78.5 & 18.5 & \textbf{67.5} & 22.8 & 41.5 & 45.8 \\
4 & 77.3 & \textbf{22.2} & 62.5 & \textbf{26.8} & 40.9 & 45.9 \\
8 & \textbf{80.9} & 18.5 & 60.0 & \textbf{26.8} & \textbf{44.4} & \textbf{46.1} \\
16 & 80.1 & 14.8 & 65.0 & 23.2 & 41.2 & 44.9 \\
32 & 77.1 & 11.1 & 65.0 & 25.4 & 39.4 & 43.6 \\
\bottomrule
\end{tabular}
\end{table}

The sliding window size $K$ primarily acts on the smoothed historical baseline $\mu_{his}$, which reflects the model's generalized capability prior to the current policy update. As shown in Table~\ref{tab:ablation_k}, PIPO consistently outperforms the baseline across varying window sizes, but exhibits a temporal trade-off. A smaller window (e.g., $K$=2 or 4) provides a less accurate estimation of the historical baseline, as it is more susceptible to the randomness of varying problem difficulty and sampling noise. On the other hand, an excessively large window (e.g., $K$=32) introduces overly stale data. Because the policy improves rapidly during RL post-training, a delayed $\mu_{his}$ consistently underestimates the model's actual capability. We find that $K$=8 serves as a robust equilibrium, effectively providing an accurate estimation of the historical capability while remaining responsive enough to track rapid capability gains.

\subsection{Robustness Across Random Seeds}

\begin{figure}[H]
  \centering
  \includegraphics[width=0.9\columnwidth]{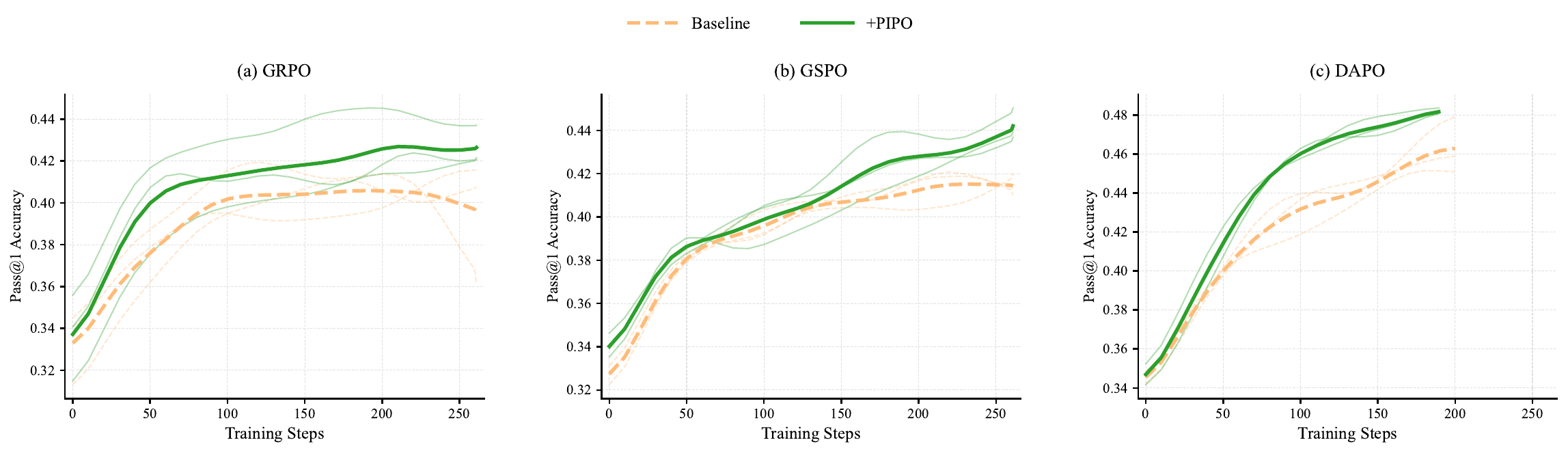}
  \caption{Training dynamics across multiple random seeds on Qwen3-4B-Base.}
  \label{fig:seed_robustness}
\end{figure}

Reinforcement learning algorithms, particularly in sparse-reward reasoning tasks, are sensitive to data sampling orders and initializations. To rigorously evaluate the stability of our proposed closed-loop mechanism, we extend our training dynamics analysis across multiple random data seeds.

We evaluate GRPO, DAPO, and GSPO alongside their PIPO-augmented variants on the Qwen3-4B-Base model using three different data seeds (6, 21, and 42). As shown in Figure~\ref{fig:seed_robustness}, the PIPO-enhanced algorithms consistently converge to a higher performance plateau across seeds. The retrospective verification effectively prevents the policy from being derailed by spurious, seed-specific noise, showing that our framework stabilizes exploration regardless of the underlying data sampling stochasticity.

\subsection{Wall-Clock Efficiency Analysis}

\begin{figure}[h]
  \centering
  \includegraphics[width=0.7\columnwidth]{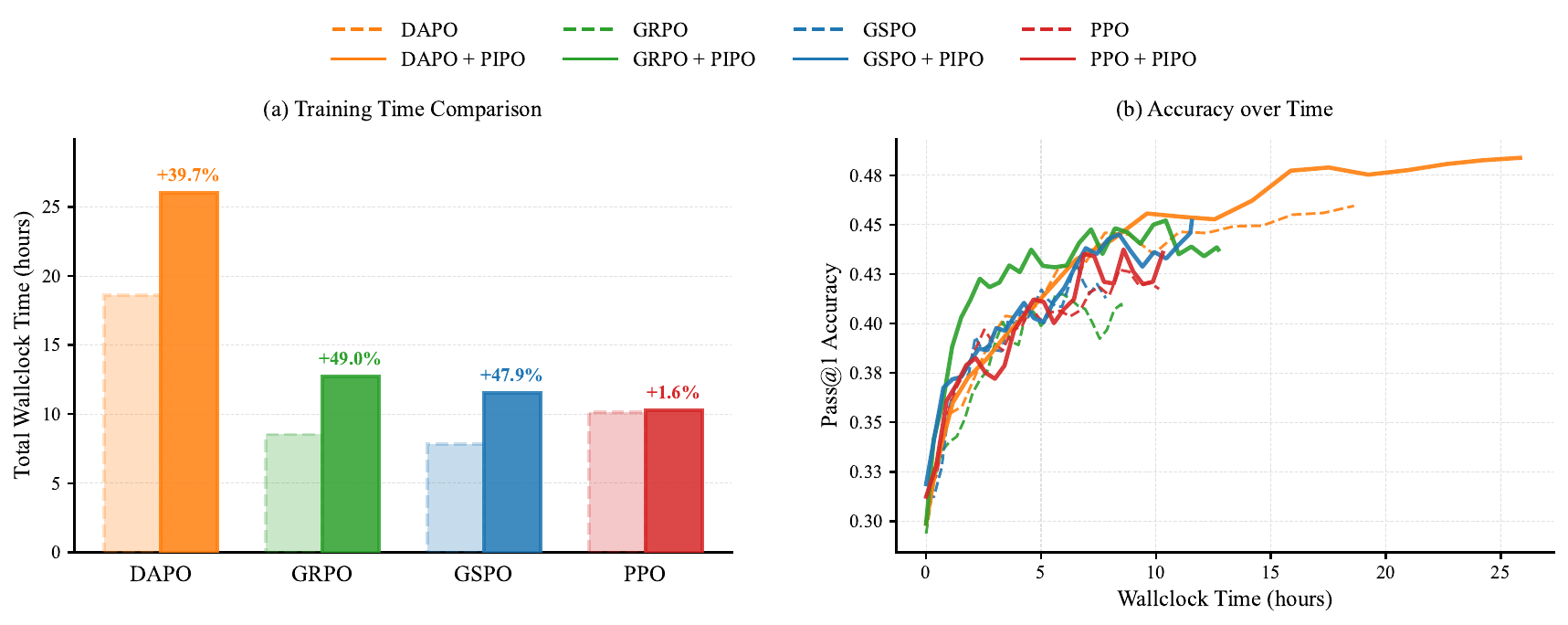}
  \caption{Computational efficiency analysis on Qwen3-4B-Base. (a) Total wall-clock time comparison. (b) Evolution of average Pass@1 accuracy with respect to wall-clock time.}
  \label{fig:efficiency}
\end{figure}

We analyze the computational efficiency in Figure~\ref{fig:efficiency}. PIPO introduces algorithm-dependent overhead: in PPO, the extra forward pass is largely amortized by the costly critic training and PPO update, yielding almost no additional wall-clock cost; in critic-free methods such as GRPO, GSPO, and DAPO, the same extra pass accounts for a larger fraction of training time and therefore leads to a more visible relative overhead. Nevertheless, \emph{PIPO consistently outperforms its open-loop counterparts in overall wall-clock efficiency}, achieving higher accuracy under comparable or moderately increased training time.

\section{Theoretical Proofs}
\label{sec:appx_theoretical}

\subsection{Proof of Theorem~\ref{prop:pirl_alignment}}
\label{sec:appx_pirl_proofs}

For notational convenience, denote the expected RL performance at iteration $t$ by 
$J_t := J_{\mathrm{RL}}(\theta_t)$. 

\begin{theorem}[Restatement of Theorem~\ref{prop:pirl_alignment}]
For a fixed initialization $\theta_0$, maximizing the ideal PIRL objective exactly maximizes the final policy performance:
\[
\arg\max_{\{\theta_t\}_{t=1}^T} \sum_{t=1}^T \mathbb{E}[ \Delta J_t ]
=
\arg\max_{\theta_T} J_{\mathrm{RL}}(\theta_T).
\]
\end{theorem}

\begin{proof}
By Definition~\ref{def:policy_improvement}, $\Delta J_t = J_t - J_{t-1}$. Summing over $T$ iterations gives the telescoping sum

\begin{align*}
\sum_{t=1}^T \Delta J_t
&= \sum_{t=1}^T (J_t - J_{t-1})  \\
&= (J_1-J_0) + (J_2-J_1) + \cdots + (J_T-J_{T-1}) \\
&= J_T - J_0.
\end{align*}

Since $J_0$ is determined by the fixed initialization $\theta_0$, it is constant with respect to the optimization trajectory $\{\theta_t\}_{t=1}^T$. Therefore
\[
\arg\max_{\{\theta_t\}_{t=1}^T} 
\sum_{t=1}^T \mathbb{E}[\Delta J_t]
=
\arg\max_{\theta_T} \mathbb{E}[J_T - J_0]
=
\arg\max_{\theta_T} J_T .
\]
This proves the result.
\end{proof}

\paragraph{Smoothed historical anchor.}
The sliding-window anchor used in PIPO can be viewed as a smoothed version of the one-step improvement signal. Define the idealized $K$-step smoothed improvement as
\[
\Delta J_t^{(K)}
:=
J_t
-
\frac{1}{K}\sum_{k=1}^{K}J_{t-k}.
\]
Assuming the policies before optimization are fixed to the initial model, their performance contributes only a constant offset. Summing the smoothed improvements gives
\begin{align*}
\sum_{t=1}^{T}\Delta J_t^{(K)}
&=
\sum_{t=1}^{T}J_t
-
\frac{1}{K}\sum_{t=1}^{T}\sum_{k=1}^{K}J_{t-k} \\
&=
\sum_{m=T-K+1}^{T}
\left(1-\frac{T-m}{K}\right)J_m
+ C_{\mathrm{init}} \\
&=
\sum_{i=0}^{K-1}
\frac{K-i}{K}J_{T-i}
+ C_{\mathrm{init}},
\end{align*}
where $C_{\mathrm{init}}$ absorbs terms determined by the fixed initialization. Thus, maximizing cumulative smoothed improvement is equivalent to maximizing a positively weighted average of the final $K$ policy performances. This justifies using the historical anchor as a variance-reduced surrogate for recent policy improvement, while the proof of Proposition~\ref{prop:approx_pirl_ascent} uses the one-step improvement $\Delta J_t$ as its idealized local counterpart.

\subsection{Proof of Proposition~\ref{prop:approx_pirl_ascent}}
\label{sec:proof_approx_pirl_ascent}

\begin{assumption}[Empirical Improvement Consistency]
\label{assump:local_improvement_consistency}
Let $\Delta J_t := J_{\mathrm{RL}}(\theta_t)-J_{\mathrm{RL}}(\theta_{t-1})$ denote the realized one-step policy improvement. The standardized feedback is sign-consistent with this realized improvement:
\[
\phi_\lambda(\xi_t)
\Delta J_t
\ge 0.
\]
\end{assumption}

This assumption is mild in the local regime considered here. The quantity $\mu_t$ is a Monte Carlo estimate of the current policy performance, while $\mu_{\mathrm{his}}$ is a smoothed empirical anchor for recent performance. Thus, $\xi_t$ is a standardized empirical estimate of policy improvement relative to this anchor; increasing the number of verifier evaluations reduces the variance of this estimate. Following the smoothed-anchor note above, the proof below uses the one-step improvement $\Delta J_t$ as the idealized local counterpart of this empirical signal.

\begin{proposition}[Restatement of Proposition~\ref{prop:approx_pirl_ascent}]
Let $\mathcal{H}_{t-1}$ denote the training history up to iteration $t-1$. For each historical learning unit $u_{t-1,i}$ with local attribution $a_{t-1,i}$, the PI reward satisfies:
\begin{enumerate}[label=(\alph*)]
\item Conditioned on $\mathcal{H}_{t-1}$ and $\mathcal{B}_{t-1}$,
\[
\mathbb{E}\!\left[\hat r^{\mathrm{PI}}_{t,i}\mid \mathcal{H}_{t-1},\mathcal{B}_{t-1}\right]
=
a_{t-1,i}\,
\mathbb{E}\!\left[\phi_\lambda(\xi_t)\mid \mathcal{H}_{t-1},\mathcal{B}_{t-1}\right].
\]
\item Under Assumption~\ref{assump:local_improvement_consistency}, a local bounded-gradient condition, and the local trust-region approximation in Section~\ref{subsec:general_modulation}, the retrospective PIPO update is locally aligned with the first-order improvement direction of the PIRL objective:
\[
\left\langle \nabla J_{\mathrm{RL}}(\theta_t), \Delta\theta_{\mathrm{PI}}\right\rangle
\ge
-\mathcal{O}(\|\theta_t-\theta_{t-1}\|).
\]
\end{enumerate}
\end{proposition}

\begin{proof}
For part (a), Eq.~\eqref{eq:meta_reward_def} gives
\[
\hat r^{\mathrm{PI}}_{t,i}
=
a_{t-1,i}\phi_\lambda(\xi_t).
\]
Conditioned on $\mathcal{H}_{t-1}$ and $\mathcal{B}_{t-1}$, the local attribution $a_{t-1,i}$ is already fixed by the previous batch and the base RL algorithm. The remaining randomness comes from the empirical improvement feedback evaluated at iteration $t$. Therefore,
\[
\mathbb{E}\!\left[\hat r^{\mathrm{PI}}_{t,i}\mid \mathcal{H}_{t-1},\mathcal{B}_{t-1}\right]
=
a_{t-1,i}\,
\mathbb{E}\!\left[\phi_\lambda(\xi_t)\mid \mathcal{H}_{t-1},\mathcal{B}_{t-1}\right].
\]

For part (b), consider the retrospective PI objective over the historical learning units $u_{t-1,i}$:
\[
\mathcal{J}_{\mathrm{PI}}(\theta)
=
\mathbb{E}\!\left[
\sum_i
w_i
\rho_i(\theta)\,
\hat r^{\mathrm{PI}}_{t,i}
\right],
\]
where $w_i$ denotes the averaging weight used by the base RL objective, and $\rho_i(\theta)$ is instantiated at the token or sequence level depending on the underlying objective. Differentiating at $\theta_t$ and substituting Eq.~\eqref{eq:meta_reward_def} gives
\[
\nabla_\theta \mathcal{J}_{\mathrm{PI}}(\theta_t)
=
\phi_\lambda(\xi_t)\,
\gamma_{t-1},
\]
where
\[
\gamma_{t-1}
=
\mathbb{E}\!\left[
\sum_i
w_i
\rho_i(\theta_t)\,
a_{t-1,i}\,
\nabla_\theta \ell_i(\theta_t)
\right].
\]
Here, $\ell_i(\theta)$ denotes the corresponding token- or sequence-level log-likelihood term. Under a local trust-region approximation, evaluating the historical learning units at $\theta_t$ only introduces local deviations from the base update direction evaluated around $\theta_{t-1}$:
\[
\gamma_{t-1}
=
g_{t-1}^{\mathrm{base}}
+
\mathcal{O}(\|\theta_t-\theta_{t-1}\|).
\]
Thus the retrospective PIPO update can be written as
\[
\Delta\theta_{\mathrm{PI}}
=
c_t g_{t-1}^{\mathrm{base}}+\varepsilon_t,
\qquad
c_t \propto \phi_\lambda(\xi_t),
\qquad
\|\varepsilon_t\|=\mathcal{O}(\|\theta_t-\theta_{t-1}\|).
\]
Because the PI learning rate is positive, $c_t$ has the same sign as $\phi_\lambda(\xi_t)$. Taking the inner product with the local RL gradient gives
\[
\left\langle \nabla J_{\mathrm{RL}}(\theta_t), \Delta\theta_{\mathrm{PI}}\right\rangle
=
c_t
\left\langle \nabla J_{\mathrm{RL}}(\theta_t), g_{t-1}^{\mathrm{base}}\right\rangle
+
\left\langle \nabla J_{\mathrm{RL}}(\theta_t), \varepsilon_t\right\rangle .
\]
It remains to relate the empirical improvement feedback to the inner product term. Write the previous base update locally as
\[
\theta_t
=
\theta_{t-1}
+
\alpha_{\mathrm{base}} g_{t-1}^{\mathrm{base}},
\qquad
\alpha_{\mathrm{base}}>0.
\]
A first-order Taylor expansion gives
\[
\Delta J_t
=
\alpha_{\mathrm{base}}
\left\langle
\nabla J_{\mathrm{RL}}(\theta_{t-1}),
g_{t-1}^{\mathrm{base}}
\right\rangle
+
\mathcal{O}(\|\theta_t-\theta_{t-1}\|^2).
\]
Under the local trust-region approximation, $\nabla J_{\mathrm{RL}}(\theta_t)$ differs from $\nabla J_{\mathrm{RL}}(\theta_{t-1})$ by a local perturbation, so
\[
\left\langle
\nabla J_{\mathrm{RL}}(\theta_t),
g_{t-1}^{\mathrm{base}}
\right\rangle
=
\frac{\Delta J_t}{\alpha_{\mathrm{base}}}
+
\mathcal{O}(\|\theta_t-\theta_{t-1}\|).
\]
Since $c_t$ has the same sign as $\phi_\lambda(\xi_t)$, Assumption~\ref{assump:local_improvement_consistency} implies that the leading term of
$c_t\left\langle \nabla J_{\mathrm{RL}}(\theta_t), g_{t-1}^{\mathrm{base}}\right\rangle$
is non-negative, up to a local perturbation of order $\mathcal{O}(\|\theta_t-\theta_{t-1}\|)$. Under a local bounded-gradient condition, the perturbation term involving $\varepsilon_t$ is also bounded by $\mathcal{O}(\|\theta_t-\theta_{t-1}\|)$. Therefore,
\[
\left\langle \nabla J_{\mathrm{RL}}(\theta_t), \Delta\theta_{\mathrm{PI}}\right\rangle
\ge
-\mathcal{O}(\|\theta_t-\theta_{t-1}\|),
\]
which proves the approximate PIRL ascent statement.
\end{proof}

\subsection{Proofs of Theorem~\ref{the:grpo_misalignment} and Corollary~\ref{cor:boundary_sensitivity}}
\label{sec:appx_grpo}

\begin{theorem}[Restatement of Theorem~\ref{the:grpo_misalignment}]
Under the above conditions, the expected GRPO update direction satisfies
\[
g_{\mathrm{GRPO}}(\theta)
=
\eta(p(q;\theta)) \cdot g_{\mathrm{ideal}},
\]
where $g_{\mathrm{ideal}} = \nabla_\theta J_{\mathrm{RL}}(\theta)$ and the scaling factor is
\[
\eta(p)
=
\frac{\sum_{k=1}^{G-1}\sqrt{k(G-k)}\binom{G}{k}p^k (1-p)^{G-k}}{G\,p(1-p)\left(1 - p^G - (1-p)^G\right)}.
\]
Thus GRPO follows a state-dependent rescaling of the ideal RL gradient.
\end{theorem}

To render the GRPO objective analytically tractable, we formally conduct our analysis under a fixed query $q$ and adopt the following standard premises: 
(i) \textbf{Non-degenerate Event:} We condition our analysis on the non-degenerate event $\mathcal{E} = \{1 \le S \le G-1\}$, where the gradient signal is non-zero.
(ii) \textbf{Local Trust Region:} We assume the policy remains close to the behavior policy ($\pi_\theta \approx \pi_{\theta_{\text{old}}}$), abstracting away the effects of the clipping mechanism.
(iii) \textbf{i.i.d. Rollouts:} Following standard parallel decoding mechanisms, the $G$ responses are sampled independently and identically given $q$. This guarantees that the number of successful responses $S$ follows a binomial distribution $S \sim \text{Binomial}(G, p)$, where $p = p(q;\theta)$ is the policy success probability.

Based on the simplifications above and the non-degenerate event $\mathcal{E} = \{1 \le S \le G-1\}$, we have the policy gradient of GRPO:
\begin{equation}
\label{eq:appx_grpo_estimator}
g_{\text{GRPO}}
=
\mathbb{E} \left[ \frac{1}{G} \sum_{i=1}^G \nabla_\theta \log \pi_\theta(y_{i}) \frac{R_i - \hat{\mu}}{\hat{\sigma}} \;\middle|\; \mathcal{E} \right].
\end{equation}

\begin{proof}[Proof of Theorem~\ref{the:grpo_misalignment}]
For any given state $S \in \mathcal{E}$, exactly $S$ responses receive a reward of $1$, and the remaining $G-S$ responses receive $0$. The group-level reward mean is therefore $\mu = \frac{S}{G}$. The empirical variance can be explicitly computed by summing over the correct and incorrect responses:
\begin{equation}\label{eq:variance_derivation}
\begin{split}
    \sigma^2 
    &= \frac{1}{G} \left[ S \left(1 - \frac{S}{G}\right)^2 + (G-S) \left(0 - \frac{S}{G}\right)^2 \right] \\
    &= \frac{S(G-S)}{G^2}.
\end{split}
\end{equation}

Substituting $\mu$ and $\sigma$ into the advantage $A_i = (R_i - \mu) / \sigma$, the advantage collapses into two discrete values for correct ($R_i=1$, denoted as $A^+$) and incorrect ($R_i=0$, denoted as $A^-$) responses:
\begin{equation}\label{eq:discrete_advantages}
A^+ = \frac{1 - S/G}{\sqrt{S(G-S)}/G} = \sqrt{\frac{G-S}{S}}, \qquad 
A^- = \frac{0 - S/G}{\sqrt{S(G-S)}/G} = -\sqrt{\frac{S}{G-S}}.
\end{equation}

Let $C = \{i : R_i = 1\}$ be the indices of the successful responses. We can decompose the marginal probability of any response $y$ by conditioning on its reward outcome $R$: $\pi_\theta(y) = \mathbb{P}(R) \cdot \pi_\theta(y \mid R)$.
Taking the gradient of the log-probability yields:
\begin{equation}
\nabla_\theta \log \pi_\theta(y) = 
\begin{cases}
\frac{\nabla_\theta p_t}{p_t} + \nabla_\theta \log \pi_\theta(y \mid R=1), & \text{if } R(y)=1, \\[6pt]
-\frac{\nabla_\theta p_t}{1-p_t} + \nabla_\theta \log \pi_\theta(y \mid R=0), & \text{if } R(y)=0.
\end{cases}
\end{equation}

We now substitute this decomposition to aggregate the gradient contributions across the sampled group:
\begin{equation}\label{eq:16}
\begin{split}
    &\sum_{i=1}^G \nabla_\theta \log \pi_\theta(y_{i}) A_i \\
    ={}& A^+ \sum_{i \in C} \nabla_\theta \log \pi_\theta(y_i) + A^- \sum_{j \notin C} \nabla_\theta \log \pi_\theta(y_j) \\
    ={}& \left( S \cdot A^+ \cdot \frac{1}{p_t} - (G-S) \cdot A^- \cdot \frac{1}{1-p_t} \right) \nabla_\theta p_t \\
    &+ A^+ \sum_{i \in C} \nabla_\theta \log \pi_\theta(y_i \mid R=1) + A^- \sum_{j \notin C} \nabla_\theta \log \pi_\theta(y_j \mid R=0).
\end{split}
\end{equation}

Substituting the explicit values of $A^+$ and $A^-$, the coefficient of $\nabla_\theta p_t$ simplifies exactly to:
\[
S \sqrt{\frac{G-S}{S}} \frac{1}{p_t} - (G-S) \left(-\sqrt{\frac{S}{G-S}}\right) \frac{1}{1-p_t} = \frac{\sqrt{S(G-S)}}{p_t(1-p_t)}.
\]

Now, we take the conditional expectation of this entire expression given a fixed number of successes $S$. Crucially, conditioned on $S$, the advantage values ($A^+, A^-$) and the marginal probability $p_t$ are strictly deterministic constants. The randomness comes entirely from the specific realization of the responses $y_i$ drawn from their respective conditional distributions $\pi_\theta(y \mid R)$. By the linearity of expectation, we can factor out the constant part:
\begin{equation}
\begin{split}
    &\mathbb{E}_{y_i \mid S} \left[ \sum_{i=1}^G \nabla_\theta \log \pi_\theta(y_{i}) A_i \right] \\
    ={}& \mathbb{E}_{y_i \mid S} \left[ \frac{\sqrt{S(G-S)}}{p_t(1-p_t)} \nabla_\theta p_t + A^+ \sum_{i \in C} \nabla_\theta \log \pi_\theta(y_i \mid R=1) + A^- \sum_{j \notin C} \nabla_\theta \log \pi_\theta(y_j \mid R=0) \right] \\
    ={}& \frac{\sqrt{S(G-S)}}{p_t(1-p_t)} \nabla_\theta p_t + A^+ \sum_{i \in C} \mathbb{E}_{y \sim \pi(\cdot \mid R=1)}\left[\nabla_\theta \log \pi_\theta(y \mid R=1)\right] + A^- \sum_{j \notin C} \mathbb{E}_{y \sim \pi(\cdot \mid R=0)}\left[\nabla_\theta \log \pi_\theta(y \mid R=0)\right].
\end{split}
\end{equation}

For any valid conditional probability distribution, the expectation of its score function is identically zero:
\[
\mathbb{E}_{y \sim \pi_\theta(\cdot \mid R)}\left[ \nabla_\theta \log \pi_\theta(y \mid R) \right] = \nabla_\theta \sum_{y} \pi_\theta(y \mid R) = \nabla_\theta 1 = 0.
\]
Therefore, the residual expectation terms vanish entirely. The expected gradient contribution conditioned purely on the state $S$ simplifies exactly to:
\[
\mathbb{E} \left[ \sum_{i=1}^G \nabla_\theta \log \pi_\theta(y_{i}) A_i \;\middle|\; S \right] = \frac{\sqrt{S(G-S)}}{p_t(1-p_t)} \nabla_\theta p_t.
\]

Taking the full expectation over $S \in \mathcal{E}$, and recognizing that $\nabla_\theta p_t = g_{\mathrm{ideal}}$, we obtain:
\begin{equation}\label{eq:grpo_expectation_form}
g_{\text{GRPO}} = \left( \frac{\mathbb{E}[\sqrt{S(G-S)} \mid \mathcal{E}]}{G \cdot p_t(1-p_t)} \right) \cdot g_{\mathrm{ideal}}.
\end{equation}

To derive the exact analytical form of the scaling factor $\eta(p)$, we expand the conditional expectation. Given the i.i.d. premise, the unconditional probability mass function of $S$ is $\mathbb{P}(S=k) = \binom{G}{k} p^k (1-p)^{G-k}$. 
The probability of the non-degenerate event $\mathcal{E}$ is the complement of the extremes ($S=0$ and $S=G$):
\[
\mathbb{P}(\mathcal{E}) = 1 - p^G - (1-p)^G.
\]
Thus, the conditional expectation can be explicitly written as:
\begin{equation}
\label{eq:eta_expansion}
\mathbb{E}[\sqrt{S(G-S)} \mid \mathcal{E}] 
= \sum_{k=1}^{G-1} \sqrt{k(G-k)} \frac{\mathbb{P}(S=k)}{\mathbb{P}(\mathcal{E})} 
= \frac{\sum_{k=1}^{G-1} \sqrt{k(G-k)} \binom{G}{k} p^k (1-p)^{G-k}}{1 - p^G - (1-p)^G}.
\end{equation}
Plugging Eq.~\eqref{eq:eta_expansion} back into Eq.~\eqref{eq:grpo_expectation_form} directly yields the final expression for $\eta(p)$ presented in Theorem~\ref{the:grpo_misalignment}. 
\end{proof}

\begin{corollary}[Restatement of Corollary~\ref{cor:boundary_sensitivity}]
At the probability boundaries ($p \to 0$ or $p \to 1$), the gradient scaling factor exhibits a symmetric singularity and diverges to infinity:
\[
\eta(p) \sim \frac{\sqrt{G-1}}{G \cdot p(1 - p)} \longrightarrow \infty.
\]
\end{corollary}

\begin{proof}[Proof of Corollary~\ref{cor:boundary_sensitivity}]
From Theorem~\ref{the:grpo_misalignment},
\[
\eta(p)=
\frac{\sum_{k=1}^{G-1}\sqrt{k(G-k)}\binom{G}{k}p^k(1-p)^{G-k}}
{G\,p(1-p)\left(1-p^G-(1-p)^G\right)} .
\]

As $p\to0$, the leading term of the numerator is $k=1$:
\[
\sum_{k=1}^{G-1}\sqrt{k(G-k)}\binom{G}{k}p^k(1-p)^{G-k}
=
\sqrt{G-1}\,G p(1-p)^{G-1}+\mathcal O(p^2).
\]

For the denominator,
\[
(1-p)^G = 1-Gp+\mathcal O(p^2),
\qquad
1-p^G-(1-p)^G = Gp+\mathcal O(p^2).
\]

Taking the ratio gives
\[
\lim_{p\to0}
\frac{\sqrt{G-1}\,G p(1-p)^{G-1}+\mathcal O(p^2)}
{Gp+\mathcal O(p^2)}
=
\sqrt{G-1}.
\]

Substituting back yields
\[
\eta(p)
=
\frac{\sqrt{G-1}+o(1)}{G\,p(1-p)}
\sim
\frac{\sqrt{G-1}}{G\,p(1-p)}.
\]

By symmetry of $\binom{G}{k}$ and $p(1-p)$, the same argument holds as
$p\to1$, implying $\eta(p)\to\infty$ at both boundaries.
\end{proof}

\subsection{Proof of Proposition~\ref{prop:rectification}}
\label{sec:appx_reconstruction}

\begin{proposition}[Restatement of Proposition~\ref{prop:rectification}]
Under bounded local updates, if the retrospective modulation uses the realized policy improvement $\Delta J_t$ as ideal feedback, PIPO can damp the leading boundary scaling factor in the group-relative case up to first-order approximation. In particular, the retrospective PI gradient remains a bounded modulation of the ideal success-probability gradient near the boundaries:
\begin{equation*}
    \mathbb{E}[\nabla_\theta \mathcal{J}_{\mathrm{PI}}]
    =
    \kappa(p)\nabla_\theta p + \mathcal{O}(\alpha_{\mathrm{std}}),
    \qquad
    \kappa(p)=\mathcal{O}(1),
    \quad \text{as } p \to 0 \text{ or } 1.
\end{equation*}
\end{proposition}

\begin{lemma}
\label{lemma:boundary_vanishing}
For any differentiable policy $\pi_\theta$ with bounded log-gradients $\|\nabla_\theta \log \pi_\theta\| \le M$ (a regularity condition typically satisfied by Layer Normalization, Gradient Clipping and other methods in modern LLMs), the realizable probability improvement for a bounded update $\|\Delta\theta\| \le \epsilon$ satisfies:
\[
\Delta p_t = \mathcal{O}(\min\{p_t, 1-p_t\}).
\]
\end{lemma}

\begin{proof}
Let $C = \{y : R(y) = 1\}$. Using the relation $\nabla p_t = \sum_{y \in C} \pi_\theta(y) \nabla \log \pi_\theta(y)$ alongside the norm bound $\|\nabla \log \pi_\theta\| \le M$, we have:
\[
\|\nabla p_t\| \le \sum_{y \in C} \pi_\theta(y) \|\nabla \log \pi_\theta(y)\| \le M \cdot p_t.
\]

Symmetry over the failure cases yields $\|\nabla p_t\| \le M \cdot (1 - p_t)$. Combining both conditions gives:
\[
\|\nabla p_t\| \le M \min(p_t, 1-p_t).
\]

For a bounded parameter update $\|\Delta\theta\| \le \epsilon$, the realizable improvement approximates as $\Delta p_t \approx (\nabla_\theta p_t)^\top \Delta\theta$. Thus, $\Delta p_t \le \epsilon M \min(p_t, 1-p_t) = \mathcal{O}(\min\{p_t, 1-p_t\})$, confirming the Lemma.
\end{proof}

\begin{proof}[Proof of Proposition~\ref{prop:rectification}]
For brevity, let $Z_{t-1} = \sum_{j=1}^G |A(y_j)|$. At iteration $t$, the Policy Improvement Reward assigns a gated signal to samples from the previous batch $\mathcal{B}_{t-1}$:
\[
    \hat{r}^{\mathrm{PI}}_{t,i} = \frac{G \cdot A(y_i)}{Z_{t-1}} \cdot \Delta J_t.
\]
Here $\Delta J_t$ denotes the realized policy improvement used as ideal feedback for the analysis. PIPO retrospectively updates the policy via an importance-weighted objective:
\[
    \mathcal{J}_{\mathrm{PI}}(\theta) = \mathbb{E}_{q, \{y_i\} \sim \mathcal{B}_{t-1}} \!\left[ \frac{1}{G} \sum_{i=1}^G \frac{\pi_\theta(y_i \mid q)}{\pi_{\theta_{t-1}}(y_i \mid q)} \hat{r}^{\mathrm{PI}}_{t,i} \right].
\]

Evaluating the exact gradient at the current policy $\theta_t$ yields:
\begin{equation}
    \nabla_\theta \mathcal{J}_{\mathrm{PI}}(\theta_t) = \frac{G \,\Delta J_t}{Z_{t-1}} \left( \frac{1}{G} \sum_{i=1}^G \frac{\pi_{\theta_t}(y_i \mid q)}{\pi_{\theta_{t-1}}(y_i \mid q)} \nabla_\theta \log \pi_{\theta_t}(y_i \mid q) A(y_i) \right).
\end{equation}

Under the local trust region assumption $\theta_t = \theta_{t-1} + \mathcal{O}(\alpha_{\text{std}})$, a first-order Taylor expansion of the importance weight and score function decouples the gradient into the standard GRPO gradient $g_{t-1}$ plus a perturbation term:
\begin{equation}
\label{eq:pi_gradient_taylor}
    \nabla_\theta \mathcal{J}_{\mathrm{PI}}(\theta_t) = \frac{G \, \Delta J_t}{Z_{t-1}} \Big[ g_{t-1} + \mathcal{O}(\alpha_{\text{std}}) \Big],
\end{equation}
where $g_{t-1} = \frac{1}{G} \sum_{i=1}^G \nabla_\theta \log \pi_{\theta_{t-1}}(y_i \mid q) A(y_i)$.

From Theorem~\ref{the:grpo_misalignment} and Eq.~\eqref{eq:discrete_advantages}, given a specific batch outcome $S$, we have $g_{t-1} = \frac{\sqrt{S(G-S)}}{G \cdot p(1-p)} \nabla_\theta p$ and $Z_{t-1} = 2\sqrt{S(G-S)}$. Substituting these back yields:
\begin{equation}
    \nabla_\theta \mathcal{J}_{\text{PI}}(\theta_t) = \frac{\Delta J_t}{2 p(1-p)} \nabla_\theta p + \mathcal{O}\left( \frac{\alpha_{\text{std}} \Delta J_t}{\sqrt{S(G-S)}} \right).
\end{equation}

We now analyze the behavior at the probability boundaries. In the sparse-reward RL setting ($R \in \{0, 1\}$), the realized policy improvement is the success-probability shift, $\Delta J_t=\Delta p_t$. By Lemma~\ref{lemma:boundary_vanishing}, $\Delta J_t=\mathcal{O}(\min\{p,1-p\})=\mathcal{O}(p(1-p))$ near the boundaries. 

For the primary term, the scalar coefficient $\frac{\Delta J_t}{p(1-p)}$ remains strictly bounded as $p \to 0$ or $1$, offsetting the leading geometric singularity. For the perturbation term, since GRPO conditions on the non-degenerate event $\mathcal{E} = \{1 \le S \le G-1\}$, the denominator strictly satisfies $\sqrt{S(G-S)} \ge \sqrt{G-1} > 0$. Thus, the perturbation is uniformly bounded by $\mathcal{O}(\alpha_{\text{std}})$.

Taking expectations, the leading coefficient remains bounded up to a step-size perturbation:
\begin{equation*}
    \mathbb{E}[\nabla_\theta \mathcal{J}_{\text{PI}}]
    =
    \kappa(p)\nabla_\theta p + \mathcal{O}(\alpha_{\text{std}}),
    \qquad
    \kappa(p)=\mathcal{O}(1).
\end{equation*}
Thus, near $p \to 0$ or $1$, the PI term damps the leading boundary singularity under the stated local assumptions. In the practical PIPO algorithm, the ideal feedback $\Delta J_t$ is replaced by the standardized empirical feedback $\phi_\lambda(\xi_t)$.
\end{proof}

\end{document}